\documentclass{article} 
\usepackage{iclr2027_conference,times}

\usepackage{amsmath}
\usepackage{amssymb}
\usepackage{amsthm}
\usepackage{mathtools}
\usepackage{booktabs}
\usepackage{graphicx}
\usepackage{longtable}
\usepackage{placeins}
\usepackage{pgfplots}
\usepgfplotslibrary{groupplots}
\usetikzlibrary{calc}
\pgfplotsset{compat=1.18}
\usepackage{xcolor}
\usepackage{sansmath}
\definecolor{plotBlue}{HTML}{2563EB}
\definecolor{plotAmber}{HTML}{D97706}
\definecolor{plotTeal}{HTML}{0D9488}
\definecolor{plotInk}{HTML}{172033}
\definecolor{plotSlate}{HTML}{475569}
\definecolor{plotMuted}{HTML}{7C8A9C}
\definecolor{plotAxis}{HTML}{BDC7D4}
\definecolor{plotGrid}{HTML}{E7ECF2}
\pgfplotsset{
  paper plot/.style={
    /tikz/font={\fontfamily{phv}\selectfont\footnotesize\sansmath},
    tick label style={font=\fontfamily{phv}\selectfont\scriptsize\sansmath,text=plotSlate},
    label style={font=\fontfamily{phv}\selectfont\footnotesize,text=plotInk},
    title style={font=\fontfamily{phv}\selectfont\footnotesize\bfseries\sansmath,text=plotInk},
    axis lines*=left,
    axis line style={draw=plotAxis,line width=0.45pt},
    tick align=outside,
    tick pos=left,
    major tick length=2pt,
    tick style={draw=plotAxis,line width=0.45pt},
    xmajorgrids=false,
    ymajorgrids=true,
    grid style={draw=plotGrid,line width=0.4pt},
    scaled ticks=false,
    /pgf/number format/1000 sep={},
    /tikz/mark size=1.7pt,
    clip marker paths=true,
    legend style={
      draw=none,fill=none,
      font=\fontfamily{phv}\selectfont\scriptsize,
      text=plotInk,
      /tikz/column sep=5pt,
      /tikz/inner sep=0pt,
      cells={anchor=west}
    },
    legend image post style={xscale=1.15}
  }
}

\usepackage{enumitem}
\usepackage{microtype}

\usepackage{hyperref}
\hypersetup{hidelinks}
\usepackage{aliascnt}
\usepackage[nameinlink,capitalise]{cleveref}

\newtheorem{theorem}{Theorem}
\newaliascnt{lemma}{theorem}
\newtheorem{lemma}[lemma]{Lemma}
\aliascntresetthe{lemma}
\newaliascnt{proposition}{theorem}
\newtheorem{proposition}[proposition]{Proposition}
\aliascntresetthe{proposition}
\newaliascnt{corollary}{theorem}

\aliascntresetthe{corollary}
\theoremstyle{definition}
\newaliascnt{definition}{theorem}

\aliascntresetthe{definition}
\crefname{lemma}{lemma}{lemmas}
\Crefname{lemma}{Lemma}{Lemmas}
\crefname{proposition}{proposition}{propositions}
\Crefname{proposition}{Proposition}{Propositions}
\crefname{corollary}{corollary}{corollaries}
\Crefname{corollary}{Corollary}{Corollaries}
\crefname{definition}{definition}{definitions}
\Crefname{definition}{Definition}{Definitions}

\usepackage{amsmath,amsfonts,bm}

\def\floor#1{\lfloor #1 \rfloor}
\def\1{\bm{1}}

\DeclareMathAlphabet{\mathsfit}{\encodingdefault}{\sfdefault}{m}{sl}
\SetMathAlphabet{\mathsfit}{bold}{\encodingdefault}{\sfdefault}{bx}{n}

\newcommand{\E}{\mathbb{E}}

\DeclareMathOperator*{\argmin}{arg\,min}

\usepackage{url}

\title{Compressing Value Predictions for Learning-Augmented Metrical Task Systems}

\author{
Sizhe Li \\
{\small School of Artificial Intelligence} \\
{\small and Automation} \\
{\small Huazhong University of Science and Technology} \\
{\small Wuhan, Hubei 430074, China} \\
\texttt{sizheree@hust.edu.cn}
\And
Yechen Li \& Kun He \\
{\small School of Computer Science} \\
{\small and Technology} \\
{\small Huazhong University of Science and Technology} \\
{\small Wuhan, Hubei 430074, China} \\
\texttt{\{lmoke,brooklet60\}@hust.edu.cn}
}

\iclrfinalcopy 
\begin{document}

\maketitle
\lhead{}

\begin{abstract}
Learning-augmented algorithms for metrical task systems (MTS) can exploit predictions of canonical dual values, but existing formulations typically require a prediction for every state. We study whether these predictions can be compressed to a small set of representative states while retaining their algorithmic value. We introduce \emph{landmark-compressed value predictions}, in which the predictor reports predicted dual values only at $m$ landmarks and the remaining values are reconstructed by a Lipschitz extension. Our algorithm achieves additive excess cost $O(T\,r(L) + \sum_t \delta_t)$, where $r(L)$ is the covering radius of the landmarks and $\delta_t$ measures prediction error up to additive shifts; local and value-dependent bounds refine this guarantee. For sparse landmark sets on unit-spaced finite lines, we prove a matching $\Omega(T r_m)$ lower bound for every randomized algorithm using fixed landmarks, even with advance access to their entire exact absolute-value table. The prediction interface also matters: on two states with one landmark, exact absolute values permit horizon-independent excess, whereas exact relative values force worst-case expected excess linear in $T$. We give PAC guarantees for learning compressed prediction tables, with efficient empirical-risk minimization for fixed landmarks. Our results connect metric coverage, prediction interfaces, and learning guarantees for compressed predictions in online MTS.

\end{abstract}

\section{Introduction}
\label{sec:introduction}

Online optimization studies problems in which inputs arrive sequentially and
each decision must be made before future inputs are known. Metrical task
systems (MTS) capture a basic tradeoff in such decisions
\citep{borodin1992MTS}: at each round, an algorithm chooses a state and
pays both its service cost and the cost of moving there from the previous
state. Tasks can favor different states over time. Moving to a cheaper state
may save service cost now, yet require costly moves later; staying avoids
movement but may incur higher service costs.

Classical competitive analysis compares an online algorithm with the offline
optimum, which knows the entire input sequence in advance
\citep{borodin1998online}. An algorithm is
$\alpha$-competitive if, on every sequence, its cost is at most $\alpha$
times the optimum, up to an input-independent additive constant; randomized
algorithms are evaluated in expectation. This provides guarantees against
arbitrary future inputs. In applications with recurring task patterns,
however, historical data can help anticipate future demand.
Learning-augmented algorithms use such predictions to improve decisions,
while analyzing how their performance depends on prediction quality
\citep{lykouris2018competitive,purohit2018improving}.
For MTS, this approach has been developed using predictions of the states to
visit \citep{antoniadis2023MTSonline}.

For MTS, canonical dual values give each state's optimal cost of serving
future tasks after the current round. Adding them to immediate movement
and service costs identifies an optimal next action when the values are exact.
\citet{coester2026dual} establish performance and learning guarantees for
full prediction tables using a value-greedy policy. Building on this policy,
we study the information lost when canonical values are available only at
a fixed set of representative states, called \emph{landmarks}. We show
that the resulting covering-radius dependence is order-optimal in the
sparse finite-line regime, even with advance access to the entire exact
absolute-value table on the landmarks. We also show that absolute and
relative value interfaces can yield qualitatively different guarantees.

A full predictor on an $n$-state metric reports $n$ predicted dual values per round;
storing and fitting an explicit table requires a parameter for every state
and prediction time. This scaling motivates reducing the representation
size. Generating predictions can also incur computational costs, motivating
algorithms that use predictions sparingly. In caching,
\citet{im2022parsimonious} reduce the number of next-arrival-time predictions,
while \citet{sadek2024} study fewer calls to an action predictor for
caching and MTS. We study the complementary question of reducing the number
of predicted values per round while keeping all states available to
the online algorithm.
The metric relates values at nearby states, suggesting redundancy. The
challenge is to control the cumulative decision cost caused by reconstruction
errors, and to determine which losses are unavoidable for any algorithm with
the same limited information. These observations lead to our central question:
\begingroup
\setlength{\topsep}{1pt}
\setlength{\partopsep}{0pt}
\begin{center}
\itshape
What is the tradeoff between prediction dimension\\*
and online performance in metrical task systems?
\end{center}
\endgroup

To study this tradeoff, we retain predicted dual values only at a fixed set
$L$ of $m$ landmarks, selected before
observing the test tasks or their predictions. The algorithm observes all
current service costs and can act at any state. Future values are
$1$-Lipschitz, allowing unreported values to be estimated by a metric extension
\citep{mcshane1934extension} and used in a value-greedy decision rule. With
exact values at every state, this rule achieves the offline optimum. We
therefore measure the extra cost introduced by compression and prediction
errors through the additive excess $\mathrm{ALG}-\mathrm{OPT}$. The greedy
rule depends only on relative values within each round.

\paragraph{Our contributions.}
\begin{itemize}[leftmargin=*,itemsep=2pt,topsep=3pt]
    \item \textbf{Compression guarantees.} For horizon $T$, our value-greedy
    policy satisfies
    \[
        \mathrm{ALG}-\mathrm{OPT}
        \le 2(T-1)r(L)+2\sum_{t=1}^{T-1}\delta_t,
    \]
    where $r(L)$ is the maximum distance to the nearest landmark and
    $\delta_t$ is the maximum absolute landmark prediction error, minimized
    over common additive shifts. Metric coverage thus controls the geometric
    compression error term; local successor-radius and value-dependent bounds
    refine this guarantee.
    \item \textbf{Matching lower bound for fixed-landmark access.} On unit-spaced finite lines in the
    sparse regime of \Cref{thm:lower-bound}, every randomized algorithm using
    at most $m$ fixed landmarks has worst-case expected excess
    $\Omega(T r_m(M))$, even given the entire exact absolute-value table on
    those landmarks in advance. Here $r_m(M)$ is the optimal $m$-covering
    radius. This matches our upper guarantee and limits every policy with
    this information access, regardless of its reconstruction rule.
    \item \textbf{Prediction-interface separation.} On two states with one
    landmark, exact absolute values permit excess at most $2r(L)$, whereas
    exact relative values force $\Omega(T r(L))$ worst-case expected excess
    for every randomized policy (\Cref{thm:raw-relative-separation}). More
    generally, exact absolute values at all but one state permit excess at
    most $2r(L)$. Thus the absolute cost scale can help other policies,
    although common shifts do not affect our greedy decisions.
    \item \textbf{Learnable representations.} Using objectives derived from
    our performance bounds, we give PAC guarantees and a polynomial-size
    linear program for fitting tables on fixed landmarks, plus generalization
    guarantees for joint landmark and table selection. The learned tables
    reuse the same time-indexed vectors across test episodes. A Citi Bike-derived
    $K$-server benchmark evaluates prediction-budget tradeoffs for mean-table
    and causal neural predictors.
\end{itemize}

\section{Problem Setup}
\label{sec:setup}

\paragraph{Metrical task systems.}
Let $(M,d)$ be a finite metric space with $n=|M|\ge2$ and diameter $D$.
The initial state $s_0\in M$ and horizon $T\ge2$ are known. An episode
$X=(c_1,\ldots,c_T)$ consists of finite service costs
$c_t:M\to\mathbb{R}_{\ge0}$. At round $t$, the algorithm observes all of
$c_t$, chooses $s_t\in M$, and pays $d(s_{t-1},s_t)+c_t(s_t)$.
Write $\mathrm{ALG}$ for its total cost and $\mathrm{OPT}$ for the minimum
cost of any path starting at $s_0$ with knowledge of $X$. We study additive
excess $\mathrm{ALG}-\mathrm{OPT}$.

\paragraph{Canonical dual values.}
Let $w_t(a)$ denote the minimum cost of serving the remaining tasks
$c_{t+1},\ldots,c_T$, starting from state $a$ \emph{after} round $t$.
No tasks remain after round $T$, so $w_T\equiv0$. The other values are
determined backward by the Bellman recursion:
\begin{equation}
    w_T\equiv0,\qquad
    w_{t-1}(a)=\min_{b\in M}
    \{d(a,b)+c_t(b)+w_t(b)\}.
    \label{eq:bellman_recursion}
\end{equation}
In particular, $w_0(s_0)=\mathrm{OPT}$.

These future costs have a dual interpretation. An offline MTS episode
is a shortest-path problem on a layered graph, where the edge
$(t-1,a)\to(t,b)$ has cost $d(a,b)+c_t(b)$.
The corresponding unit-flow LP has dual potentials $u_t(a)$.
With $u_T\equiv0$, the dual maximizes $u_0(s_0)$ subject to
$u_{t-1}(a)\le d(a,b)+c_t(b)+u_t(b)$ for all $t,a,b$
\citep{coester2026dual}.
The Bellman values satisfy these constraints and attain $\mathrm{OPT}$,
so they form an optimal dual solution. The recursion specifies these
values uniquely; we use this canonical solution as our prediction target.

We write
$(B_cf)(a)=\min_b\{d(a,b)+c(b)+f(b)\}$ for the Bellman operator.
The triangle inequality implies that each $w_t$ is $1$-Lipschitz,
i.e., $|w_t(a)-w_t(b)|\le d(a,b)$ (\Cref{lem:bellman-properties},
Appendix~\ref{app:bellman-properties}). Thus a landmark's value
constrains values at nearby states.

\paragraph{Landmarks and predictions.}
Fix a nonempty landmark set $L\subseteq M$, with $|L|=m$.
The covering radius $r(L)$ is the largest distance from a state to its
nearest landmark, and $r_m(M)$ is the smallest radius achievable with at
most $m$ landmarks:
\begin{equation}
    d(x,L)=\min_{\ell\in L}d(x,\ell),\qquad
    r(L)=\max_{x\in M}d(x,L),\qquad
    r_m(M)=\min_{\substack{S\subseteq M\\1\le |S|\le m}}r(S).
    \label{eq:covering_radii}
\end{equation}
Before choosing $s_t$ at each round $1\le t<T$, the algorithm receives
predictions $v_t\in\mathbb{R}^{L}$ of $w_t|_L$, with one predicted value per
landmark; $w_T\equiv0$ is known exactly.
Our value-greedy policy in \Cref{sec:upper_bound} depends only on relative
values within each round.
With $\|z\|_{\mathrm{sp}}=\max z-\min z$, we therefore measure prediction
error modulo a common additive shift:
\begin{equation}
    \delta_t=\inf_{a\in\mathbb{R}}
    \max_{\ell\in L}|v_t(\ell)-w_t(\ell)-a|
    =\tfrac12\|v_t-w_t|_L\|_{\mathrm{sp}}.
    \label{eq:landmark_prediction_error}
\end{equation}
The span identity is proved in Appendix~\ref{app:bellman-gaps}.

\paragraph{Prediction access model.}
The landmark set $L$ is selected before any service costs or predictions from the
episode are observed and remains fixed throughout the episode. We distinguish
two interfaces for exact values: the \emph{Raw} interface reveals $w_t|_L$,
while the \emph{Relative} interface reveals $w_t(\ell)-w_t(\ell_0)$ for a
fixed anchor $\ell_0\in L$. These interfaces can yield different guarantees
for general policies, as shown in \Cref{thm:raw-relative-separation}.

\section{Compressing Value Predictions}
\label{sec:upper_bound}

We quantify the cost of replacing full-state value predictions by
landmark predictions. The resulting excess is controlled by metric
coverage and landmark prediction error.

\paragraph{Reconstructing from landmarks.}
The canonical dual values $w_t$ (optimal remaining costs) are
$1$-Lipschitz with respect to $d$.
Thus each exact landmark value gives an upper bound
$w_t(x)\le w_t(\ell)+d(x,\ell)$, and taking the minimum over landmarks
gives the tightest of these bounds.  For landmark predictions $v_t$ at
round $t<T$, define the reconstruction operator $E_L$ and the
reconstructed values $\widehat w_t$ by the metric envelope
\begin{equation}
    \widehat w_t(x)
    \coloneqq
    (E_L v_t)(x)
    =
    \min_{\ell\in L}
    \bigl\{v_t(\ell)+d(x,\ell)\bigr\}.
    \label{eq:metric_extension}
\end{equation}
With exact landmark predictions, the reconstructed values upper-bound
$w_t$; arbitrary landmark predictions need not preserve this property.
The value-greedy policy selects
\begin{equation}
    s_t \in
    \arg\min_{x\in M}
    \left\{
        d(s_{t-1},x)+c_t(x)+\widehat w_t(x)
    \right\}.
    \label{eq:landmark_greedy}
\end{equation}
At the terminal round we set $\widehat w_T=w_T\equiv0$.
A common additive shift of $v_t$ produces the same shift in $E_Lv_t$, so
the decisions made by~\eqref{eq:landmark_greedy} depend only on relative
landmark predictions.

The reconstruction error is controlled by the distance to the nearest
landmark and the accuracy of the landmark predictions.

\begin{lemma}[Landmark reconstruction]
\label{lem:reconstruction}
For exact landmark values,
\begin{equation}
    0
    \le
    E_L(w_t|_L)(x)-w_t(x)
    \le
    2d(x,L)
    \le
    2r(L).
    \label{eq:exact_reconstruction}
\end{equation}
For arbitrary landmark predictions with per-round landmark prediction
error $\delta_t$ defined in~\eqref{eq:landmark_prediction_error},
\begin{equation}
    \bigl\|\widehat w_t-w_t\bigr\|_{\mathrm{sp}}
    \le
    2r(L)+2\delta_t .
    \label{eq:reconstruction_span}
\end{equation}
\end{lemma}

The exact bound~\eqref{eq:exact_reconstruction} follows directly from
Lipschitzness: every landmark gives an upper bound on $w_t(x)$, while the
nearest landmark is at most $2d(x,L)$ above it.  For arbitrary landmark
predictions, let $a_t$ and $b_t$ be the minimum and maximum of
$v_t(\ell)-w_t(\ell)$ over $\ell\in L$.  Taking minima in
\eqref{eq:metric_extension} preserves these bounds, so for every $x$,
\[
    a_t\le E_Lv_t(x)-E_L(w_t|_L)(x)\le b_t.
\]
Hence replacing exact landmark values by predictions adds at most
$b_t-a_t=2\delta_t$ to the reconstruction error span, proving
\eqref{eq:reconstruction_span} by the span triangle inequality.
Appendix~\ref{app:upper-bound-details} gives the full reconstruction argument.

Our main upper bound follows by relating the reconstruction error span
to the Bellman gap at each round.

\begin{theorem}[Compressed value prediction]
\label{thm:global_upper_bound}
For any fixed landmark set $L$ and arbitrary landmark predictions
$\{v_t\}_{t<T}$, the value-greedy policy in~\eqref{eq:landmark_greedy}
satisfies
\begin{equation}
    \mathrm{ALG}
    \le
    \mathrm{OPT}
    +2(T-1)r(L)
    +2\sum_{t=1}^{T-1}\delta_t .
    \label{eq:main_upper_bound}
\end{equation}
In particular, with exact landmark predictions,
\[
    \mathrm{ALG}-\mathrm{OPT}
    \le 2(T-1)r(L).
\]
\end{theorem}

\paragraph{Proof sketch.}
Let $e_t=\widehat w_t-w_t$ denote the reconstruction error, and let
$x_t^\star$ be a Bellman-optimal successor from the algorithm's current
state $s_{t-1}$.  Define the
Bellman gap incurred by the algorithm at time $t$ as
\[
    \Delta_t
    =
    d(s_{t-1},s_t)+c_t(s_t)+w_t(s_t)
    -w_{t-1}(s_{t-1}).
\]
These gaps are nonnegative and telescope to
$\mathrm{ALG}-\mathrm{OPT}=\sum_t\Delta_t$.
Since $s_t$ minimizes the predicted objective in
\eqref{eq:landmark_greedy}, whereas $x_t^\star$ minimizes the true Bellman
objective,
\begin{align}
    \Delta_t
    &\le
    e_t(x_t^\star)-e_t(s_t)
    \nonumber\\
    &\le
    \|e_t\|_{\mathrm{sp}}.
    \label{eq:bellman_gap_span}
\end{align}
The reconstructed values at round $T$ are exact, so $\Delta_T=0$.  Summing
\eqref{eq:bellman_gap_span} over $t<T$ and applying
Lemma~\ref{lem:reconstruction} proves~\eqref{eq:main_upper_bound}.
The complete argument is given in
Appendix~\ref{app:upper-bound-details}.

\paragraph{Interpretation.}
The bound separates the geometric compression error term $2(T-1)r(L)$
from the prediction error term $2\sum_{t<T}\delta_t$.
With exact landmark predictions and landmarks attaining the optimal
covering radius $r_m(M)$, the excess cost is at most $2(T-1)r_m(M)$.
Thus metric covering complexity controls the prediction budget sufficient
for a target excess cost.

\paragraph{Reconstruction distortion.}
The geometric compression error bound $2r(L)$ depends only on landmark
coverage.  To account for the values of an episode $X$, define its
reconstruction distortion as
\[
    \kappa_t(L;X)
    \coloneqq\|E_L(w_t^X|_L)-w_t^X\|_{\mathrm{sp}}
    \le 2r(L).
\]
The same Bellman-gap argument gives
\[
    \mathrm{ALG}-\mathrm{OPT}
    \le\sum_{t<T}\bigl[\kappa_t(L;X)+2\delta_t\bigr],
\]
as proved in Appendix~\ref{app:distortion-upper-bound}.
The distortion can vanish even for a large covering radius: if
$w_t^X(x)=d(x,a)$ and $L=\{a\}$, then $\kappa_t(L;X)=0$.
Since $\kappa_t$ depends only on $L$ and the episode's exact values, it
can be computed from offline training labels independently of the
predictions and online trajectory.  \Cref{sec:learning} combines this
distortion with the prediction error in its joint landmark and table
selection objective.

The next section asks whether the worst-case dependence on $r_m(M)$ is
unavoidable.

\section{Lower Bounds and Prediction Interfaces}
\label{sec:lower-bound}

We now show that the dependence on the covering radius in our upper bound is unavoidable when access to future values is restricted to fixed landmarks. The lower bound already holds in a stronger prediction model than required by our algorithm: after choosing a fixed landmark set $L$, the online algorithm may receive the exact absolute values $\{w_t(\ell)\}_{\ell\in L}$, and may even inspect the entire table of landmark values before the episode starts. We prove the result on the finite line $M=\{0,\ldots,n-1\}$ with $d(x,y)=|x-y|$, using the optimal covering radius $r_m(M)$ defined in~\eqref{eq:covering_radii}.

\begin{theorem}[Lower bound for fixed landmarks]
\label{thm:lower-bound}
Let $M=\{0,\ldots,n-1\}$ be the unit-spaced line. If $n-1\ge 16m$ and $T\ge 4$, then for every randomized online algorithm using at most $m$ fixed landmarks, there exists an oblivious task sequence such that
\[
\E[\mathrm{ALG}]-\mathrm{OPT}\ge \frac{1}{32}\floor{\frac{T}{4}}\,r_m(M).
\]
The bound holds even when the entire table of exact absolute landmark values is revealed after $L$ is selected. Together with the $O(T r_m(M))$ upper bound, this gives minimax excess $\Theta(T r_m(M))$ for access to exact values at fixed landmarks in this sparse finite-line regime.
\end{theorem}

\paragraph{Proof idea.}
Each four-round gadget hides which of two states is favorable inside one
of $2m$ disjoint intervals. At least half of these intervals contain no
landmark. On such an interval, the algorithm cannot identify the favorable
state, even from the entire landmark-value table, and incurs expected
excess proportional to the interval's length. Repeating the gadget makes
these losses accumulate. We sample all intervals and hidden signs
independently of the algorithm's randomness, including its choice of $L$.
Appendix~\ref{app:lower-bound-details} gives the complete construction.

Let $D=n-1$ and $h=\floor{D/(8m)}$. Place $2m$ disjoint open intervals
$I_i=(4hi,4h(i+1))$, $i=0,\ldots,2m-1$, on the line. Each gadget
independently samples one interval uniformly, so, conditional on any
realized $L$, its interval is unhit with probability at least $1/2$.
Within a selected interval $I=(p,p+4h)$, set
$a=p+h$, $z=p+2h$, and $b=p+3h$.
An independent fair sign $\sigma\in\{-1,+1\}$ determines which candidate
has the smaller future value. We construct a nonnegative $1$-Lipschitz
function $f_\sigma$ with a common baseline $C$ such that
\[
f_\sigma(a)=C-\sigma h,\qquad f_\sigma(b)=C+\sigma h,
\qquad f_\sigma(x)=C\quad(x\notin I).
\]
The gadget has four stages: enter $z$, decide between the candidates,
reveal $\sigma$, and reset to the common origin $o=s_0$.
Finite service penalties make the enter and reset actions uniquely
Bellman-optimal. The decision task $c^{\mathrm{dec}}$ costs zero at
$a,b$ and outside $I$, and penalizes all other interior states.
The reveal and reset tasks realize $f_\sigma$ as the values after the
decision task, up to the optimal cost of subsequent gadgets.
\Cref{fig:lower-bound-gadget} shows how the favorable candidate changes
while landmark values remain identical.

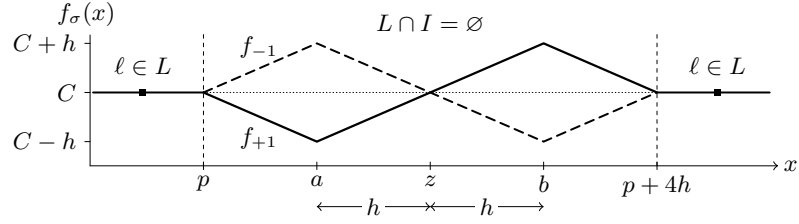
\begin{figure}[t]
    \centering
\begin{tikzpicture}[
    x=1cm,y=1cm,
    font=\normalfont\small,
    draw=black,text=black,
    line cap=round,line join=round
]
\path[use as bounding box] (0,0) rectangle (10.8,2.8);

\draw[line width=0.35pt] (1.00,0.65) -- (1.00,2.38);
\draw[->,line width=0.35pt] (1.00,0.65) -- (10.10,0.65);
\node[anchor=west,inner sep=2pt] at (10.10,0.65) {$x$};
\node[anchor=south,inner sep=2pt] at (0.95,2.42) {$f_\sigma(x)$};
\foreach \level/\ticklabel in {0.97/{C-h},1.62/C,2.27/{C+h}} {
    \draw[line width=0.35pt] (0.93,\level) -- (1.00,\level);
    \node[anchor=east,inner sep=3pt] at (0.93,\level) {$\ticklabel$};
}
\foreach \edge in {2.50,8.50} {
    \draw[line width=0.3pt,dash pattern=on 1.4pt off 1.8pt]
        (\edge,0.65) -- (\edge,2.39);
}
\node[inner sep=1pt] at (5.50,2.53) {$L\cap I=\varnothing$};
\draw[line width=0.25pt,densely dotted] (2.50,1.62) -- (8.50,1.62);

\draw[line width=0.75pt] (1.04,1.62) -- (2.50,1.62);
\draw[line width=0.75pt] (8.50,1.62) -- (9.99,1.62);
\draw[line width=0.85pt]
    (2.50,1.62) -- (4.00,0.97) -- (5.50,1.62)
    -- (7.00,2.27) -- (8.50,1.62);
\draw[line width=0.75pt,dash pattern=on 3.5pt off 2.3pt]
    (2.50,1.62) -- (4.00,2.27) -- (5.50,1.62)
    -- (7.00,0.97) -- (8.50,1.62);
\node[anchor=north,inner sep=1pt] at (3.23,1.20) {$f_{+1}$};
\node[anchor=south,inner sep=1pt] at (3.23,2.03) {$f_{-1}$};

\foreach \landmark in {1.70,9.30} {
    \fill (\landmark-0.045,1.575) rectangle (\landmark+0.045,1.665);
    \node[anchor=south,inner sep=2pt] at (\landmark,1.77) {$\ell\in L$};
}

\foreach \pos/\ticklabel in {2.50/p,4.00/a,5.50/z,7.00/b,8.50/{p+4h}} {
    \draw[line width=0.35pt] (\pos,0.60) -- (\pos,0.70);
    \node[anchor=north,inner sep=2pt] at (\pos,0.60) {$\ticklabel$};
}
\draw[<->,line width=0.35pt] (4.00,0.11) -- (5.50,0.11)
    node[midway,fill=white,inner sep=1.5pt] {$h$};
\draw[<->,line width=0.35pt] (5.50,0.11) -- (7.00,0.11)
    node[midway,fill=white,inner sep=1.5pt] {$h$};
\end{tikzpicture}
    \caption{Hidden value profiles on an unhit interval: identical landmark
    values, opposite favorable candidates. The common suffix $K_j$ in
    $w_{4j-2}=f_\sigma+K_j$ is omitted. Lines interpolate discrete-state
    values; the text establishes whole-table indistinguishability.}
    \label{fig:lower-bound-gadget}
\end{figure}

To hide $\sigma$ from the \emph{entire} table, we must also prevent it
from changing earlier absolute values through the gadget's optimal cost.
Writing $g_\sigma=B_{c^{\mathrm{dec}}}f_\sigma$, the construction ensures
\[
g_\sigma(z)=C,\qquad g_\sigma(x)=C\quad(x\notin I).
\]
Thus the optimal cost of a complete gadget starting at $o$ is
$d(o,z)+C$, independently of its sign. The optimal cost $K_j$ of all
gadgets after gadget $j$ is consequently independent of every hidden
sign. Within gadget $j$, the only value functions that can depend on its
sign are $g_\sigma+K_j$ after the enter task and $f_\sigma+K_j$ after
the decision task. Both equal $C+K_j$ at every landmark when
$L\cap I=\varnothing$, and no sign information propagates to earlier
values. Since the observed tasks through the decision stage are also
independent of $\sigma$, it remains a fair hidden bit at that stage.

If the algorithm skips the enter action to $z$, its Bellman gap is
already at least $h$. Otherwise, both candidates are distance $h$ from
$z$: the favorable one has gap zero and the other has gap $2h$, giving
expected gap $h$ for either choice. Other actions cannot avoid this
loss: an exterior state has gap at least $2h$, and the penalty at any
other interior state gives gap at least $h$. Hence every randomized
decision rule satisfies
\[
\E\!\left[\Delta_{\mathrm{in}}+\Delta_{\mathrm{dec}}
    \mid L\cap I=\varnothing\right]\ge h.
\]
All remaining Bellman gaps are nonnegative. Since each interval is
unhit with probability at least $1/2$, each gadget contributes expected
excess at least $h/2$. Concatenating $N=\floor{T/4}$ gadgets and
appending zero tasks to reach horizon $T$ gives
\[
\E[\mathrm{ALG}-\mathrm{OPT}]\ge\frac{Nh}{2}
    \ge\frac{N}{32}\,r_m(M),
\]
where $h\ge D/(16m)\ge r_m(M)/16$ in the stated sparse regime.
Here the expectation includes both input and algorithm randomness.
Averaging over inputs fixes an oblivious task sequence with the same
lower bound in expectation over the algorithm's randomness, proving
\Cref{thm:lower-bound}. \hfill$\square$

The theorem concerns fixed landmarks in the sparse finite-line regime;
adaptive queries and arbitrary encodings of future information are outside
its scope.

\paragraph{Absolute versus relative values.}
Outside the sparse regime, the Raw interface can support horizon-independent
excess. The following separation uses the same metric and landmark budget
for both interfaces.

\begin{theorem}[Raw versus Relative interfaces]
\label{thm:raw-relative-separation}
On a two-state metric of distance $a>0$ with one fixed landmark, round-by-round
access to exact Raw values permits a deterministic policy with excess at
most $2a$ on every episode, without being supplied $\mathrm{OPT}$. For every
randomized policy using the exact Relative interface, there is an oblivious episode
with
\begin{equation}
    \E[\mathrm{ALG}]-\mathrm{OPT}
       \ge\frac a2\left\lfloor\frac T4\right\rfloor.
    \label{eq:relative-lower-bound}
\end{equation}
\end{theorem}

The Raw policy initializes a budget $U\in[\mathrm{OPT},\mathrm{OPT}+2a]$
using reconstructed values at the first round. The unspent budget always
covers the true remaining cost: if the landmark's movement, service, and
exact remaining cost fit, the policy selects the landmark; otherwise, Bellman
optimality makes the sole unreported state feasible. With one landmark,
the Relative value is always zero, so the two-state construction can hide
a fresh favorable state in each four-round gadget.
Appendix~\ref{app:information-interfaces} gives both proofs and extends the
Raw guarantee to $2r(L)$ whenever only one state value is unreported.

Absolute values therefore carry information that general policies can use
across rounds, even though common shifts do not affect value-greedy decisions.
\Cref{thm:lower-bound} shows that the sparse finite-line obstruction persists
even when this information is supplied for every round in advance.

\section{Learning Landmarks and Prediction Tables}
\label{sec:learning}

We learn from $N_{\mathrm{train}}$ independent episodes drawn from $\mathcal D$;
tasks within an episode may be dependent. Offline Bellman recursion supplies
exact labels. Set $\tau=T-1$. A \emph{fixed prediction table} reuses the same
$v_t$ at time $t$ in every test episode.

\paragraph{Learning a fixed table.}
Fix $L$ and an anchor $\ell_0\in L$. The cost bound suggests the span risk
\begin{equation}
    q_t^X(\ell)=w_t^X(\ell)-w_t^X(\ell_0),\qquad
    R_L(v)=\E_{X\sim\mathcal D}\!\left[
        \frac1\tau\sum_{t=1}^{\tau}\|v_t-q_t^X\|_{\mathrm{sp}}
    \right].
    \label{eq:learning_risk}
\end{equation}
Each summand equals $2\delta_t$. Without losing decisions or increasing
span loss, we restrict to the compact class $\mathcal V_L$ of
$1$-Lipschitz tables anchored by $v_t(\ell_0)=0$
(Appendix~\ref{app:envelope-closure}). Empirical risk minimization (ERM)
replaces the expectation by the training average and is a polynomial-size
linear program for fixed $L$ (Appendix~\ref{app:learning-lp}).

\begin{theorem}[Learning a fixed landmark table]
\label{thm:fixed_table_learning}
For $\varepsilon>0$ and $\beta\in(0,1)$, if
\begin{equation}
    N_{\mathrm{train}}\ge\frac{32D^2}{\varepsilon^2}
    \left[(m-1)\log\!\left(1+\frac{32D}{\varepsilon}\right)
          +\log\frac{2\tau}{\beta}\right],
    \label{eq:learning_sample_size}
\end{equation}
then, with probability at least $1-\beta$ over training episodes, any ERM
$\widehat v\in\mathcal V_L$ satisfies
$R_L(\widehat v)\le\inf_{v\in\mathcal V_L}R_L(v)+\varepsilon$ and
\begin{equation}
    \E_X[\mathrm{ALG}_{L,\widehat v}-\mathrm{OPT}]
    \le\tau\left[2r(L)+\inf_{v\in\mathcal V_L}R_L(v)+\varepsilon\right].
    \label{eq:learning_cost_bound}
\end{equation}
\end{theorem}
The expectation is over a fresh test episode. For fixed landmarks, the
dimension term in~\eqref{eq:learning_sample_size} is $m-1$, rather than
$n-1$ for $L=M$. Each training episode supplies labels at every prediction
time; uniform convergence is established per time and then averaged,
with a union bound contributing only $\log\tau$. The accuracy $\varepsilon$
concerns time-averaged span loss; total statistical excess $\varepsilon_{\rm tot}$
requires $\varepsilon=\varepsilon_{\rm tot}/\tau$
(proof in Appendix~\ref{app:fixed-table-proof}).

\paragraph{Selecting landmarks.}
Unlike fitting $v$ for fixed $L$, selecting $L$ changes both reconstruction
distortion and prediction error. We therefore minimize the empirical version of
\begin{equation}
    J(L,v)=\E_X\!\left[\frac1\tau\sum_{t=1}^{\tau}
        \bigl(\kappa_t(L;X)+\|v_t-q_t^X\|_{\mathrm{sp}}\bigr)\right],
    \label{eq:joint_learning_objective}
\end{equation}
which bounds expected excess by $\tau J(L,v)$.
The same LP applies for each fixed $L$, but joint search is combinatorial.
Appendix~\ref{app:joint-learning-proof} gives the joint fixed-table guarantee, including
empirical optimization error. Alternatively, independent validation can
select among frozen training-fitted candidates by actual excess, with a
guarantee relative to the best candidate (Appendix~\ref{app:validation-proof}).

\section{Experiments}
\suppressfloats[t]
\label{sec:experiments}

\paragraph{Setup.}
We derive a finite-state $K$-server benchmark from Citi Bike trip starts
\citep{citibikeData}. Each day is a separate episode with 96 15-minute
slots or 1,440 1-minute slots. We map trip starts to ten ordered latitude bins and
replace each nonempty slot by a single request to the bin with the most
starts; empty slots require no service. A configuration places $K$ servers
in distinct bins, and serving a request requires a server at its bin.
The cost of changing configurations is the minimum total distance the
servers must move, measured in bin-index units. We evaluate
$K\in\{3,4,5,6\}$, giving $n=\binom{10}{K}$ configurations.

At each slot, we predict the canonical dual value for each configuration,
namely its optimal remaining cost after the current slot. We retain these
predictions at $m=|L|$ landmark configurations and reconstruct the remaining
values using the McShane extension before the value-greedy decision.
Thus the horizontal axis is $100m/n$: for $K=4$, 20\% means retaining
$42$ of $210$ predicted values per slot.
Greedy distortion-based selection (\emph{Distortion}) reduces the
reconstruction-error span (maximum minus minimum) on exact training-day
dual values computed by dynamic programming. It targets the distortion
term in \eqref{eq:joint_learning_objective}; predictors are fitted separately.
\emph{Geometric} chooses configurations farthest from
their nearest selected landmark; \emph{Random} uses nested random subsets.
Landmark sets are selected offline and fixed across slots and evaluation
days. The three methods share the same
predictor and decision rule. \emph{Full} uses all $n$ predicted values
and retains the predictor's errors. Our default mean predictor averages
exact training-day dual values by clock slot; the GRU \citep{cho2014learning} uses
observed requests and calendar features. All predictors yield spatially
$1$-Lipschitz outputs, enforced structurally for neural models.
Work Function Algorithm (WFA) \citep{koutsoupias1995kserver} and
Double Coverage (DC) \citep{chrobak1991server} are prediction-free online baselines.

After chronological validation, we refit predictors on 2023--2024 and
evaluate on all 365 days of 2025. We report mean daily ALG/OPT, averaging over seeds within
each day before averaging across days; OPT is computed by exact dynamic
programming.
Protocol details are in
\Cref{app:citibike-protocol,app:citibike-statistics}.

\begin{figure}[t]
    \centering
    \begin{tikzpicture}
\begin{groupplot}[group style={group size=3 by 1,horizontal sep=0.75cm},paper plot,ylabel={Mean daily ALG/OPT},group/ylabels at=edge left,width=0.355\linewidth,height=4.55cm,legend columns=5,legend style={at={(1.70,1.30)},anchor=south}]
\nextgroupplot[title={(a) Mean, 15 min},xmin=0,xmax=103,xtick={5,20,60,100},ymin=0.96,ymax=3.3]
\addplot[color=plotSlate,line width=0.85pt,dash pattern=on 4pt off 2.5pt,no marks,line cap=round,line join=round,forget plot] coordinates {(0.000000,1.1524) (100.000000,1.1524)};

\addplot[color=plotMuted,line width=0.8pt,dash pattern=on 1pt off 2pt,no marks,line cap=round,line join=round,forget plot] coordinates {(0.000000,1.7067) (100.000000,1.7067)};

\addplot[color=plotTeal,line width=0.95pt,mark=triangle*,mark size=2.0pt,mark options={solid,fill=white,line width=0.9pt},line cap=round,line join=round,forget plot] coordinates {(4.761905,1.7373) (10.000000,1.7373) (20.000000,2.9478) (30.000000,1.1528) (40.000000,1.1524) (60.000000,1.1524) (80.000000,1.1524) (100.000000,1.1524)};

\addplot[color=plotAmber,line width=0.95pt,mark=square*,mark size=1.7pt,mark options={solid,fill=white,line width=0.9pt},line cap=round,line join=round,forget plot] coordinates {(4.761905,1.9350) (10.000000,1.6483) (20.000000,1.5396) (30.000000,1.4614) (40.000000,1.3559) (60.000000,1.2667) (80.000000,1.1585) (100.000000,1.1524)};

\addplot[color=plotBlue,line width=1.2pt,mark=*,mark size=1.8pt,mark options={solid,fill=plotBlue,draw=white,line width=0.35pt},line cap=round,line join=round,forget plot] coordinates {(4.761905,1.8126) (10.000000,1.3311) (20.000000,1.2043) (30.000000,1.1527) (40.000000,1.1526) (60.000000,1.1526) (80.000000,1.1524) (100.000000,1.1524)};

\addlegendimage{color=plotBlue,line width=1.2pt,mark=*,mark size=1.8pt,mark options={solid,fill=plotBlue,draw=white,line width=0.35pt}}
\addlegendentry{Distortion}
\addlegendimage{color=plotAmber,line width=0.95pt,mark=square*,mark size=1.7pt,mark options={solid,fill=white,line width=0.9pt}}
\addlegendentry{Random}
\addlegendimage{color=plotTeal,line width=0.95pt,mark=triangle*,mark size=2.0pt,mark options={solid,fill=white,line width=0.9pt}}
\addlegendentry{Geometric}
\addlegendimage{color=plotSlate,line width=0.85pt,dash pattern=on 4pt off 2.5pt,no marks}
\addlegendentry{Full}
\addlegendimage{color=plotMuted,line width=0.8pt,dash pattern=on 1pt off 2pt,no marks}
\addlegendentry{WFA}

\nextgroupplot[title={(b) Mean, 1 min},xmin=0,xmax=103,xtick={5,20,60,100},ymin=0.96,ymax=6.6]
\addplot[color=plotSlate,line width=0.85pt,dash pattern=on 4pt off 2.5pt,no marks,line cap=round,line join=round,forget plot] coordinates {(0.000000,1.1788) (100.000000,1.1788)};

\addplot[color=plotMuted,line width=0.8pt,dash pattern=on 1pt off 2pt,no marks,line cap=round,line join=round,forget plot] coordinates {(0.000000,1.6160) (100.000000,1.6160)};

\addplot[color=plotTeal,line width=0.95pt,mark=triangle*,mark size=2.0pt,mark options={solid,fill=white,line width=0.9pt},line cap=round,line join=round,forget plot] coordinates {(4.761905,1.2091) (10.000000,1.1711) (20.000000,1.1895) (30.000000,1.1895) (40.000000,1.1788) (60.000000,1.1788) (80.000000,1.1788) (100.000000,1.1788)};

\addplot[color=plotAmber,line width=0.95pt,mark=square*,mark size=1.7pt,mark options={solid,fill=white,line width=0.9pt},line cap=round,line join=round,forget plot] coordinates {(4.761905,6.3123) (10.000000,4.6044) (20.000000,3.4129) (30.000000,3.4609) (40.000000,3.0237) (60.000000,1.9271) (80.000000,2.0471) (100.000000,1.1788)};

\addplot[color=plotBlue,line width=1.2pt,mark=*,mark size=1.8pt,mark options={solid,fill=plotBlue,draw=white,line width=0.35pt},line cap=round,line join=round,forget plot] coordinates {(4.761905,1.1725) (10.000000,1.1776) (20.000000,1.1786) (30.000000,1.1786) (40.000000,1.1787) (60.000000,1.1788) (80.000000,1.1788) (100.000000,1.1788)};

\nextgroupplot[title={(c) GRU, 15 min},xmin=3,xmax=32,xtick={5,10,20,30},ymin=0.96,ymax=3.3]
\addplot[color=plotSlate,line width=0.85pt,dash pattern=on 4pt off 2.5pt,no marks,line cap=round,line join=round,forget plot] coordinates {(3.000000,1.1378) (32.000000,1.1378)};

\addplot[color=plotMuted,line width=0.8pt,dash pattern=on 1pt off 2pt,no marks,line cap=round,line join=round,forget plot] coordinates {(3.000000,1.7067) (32.000000,1.7067)};

\addplot[color=plotTeal,line width=0.95pt,mark=triangle*,mark size=2.0pt,mark options={solid,fill=white,line width=0.9pt},line cap=round,line join=round,forget plot] coordinates {(4.761905,1.7197) (10.000000,1.7260) (20.000000,2.8774) (30.000000,1.1510)};

\addplot[color=plotAmber,line width=0.95pt,mark=square*,mark size=1.7pt,mark options={solid,fill=white,line width=0.9pt},line cap=round,line join=round,forget plot] coordinates {(4.761905,1.8932) (10.000000,1.6117) (20.000000,1.4633) (30.000000,1.3489)};

\addplot[color=plotBlue,line width=1.2pt,mark=*,mark size=1.8pt,mark options={solid,fill=plotBlue,draw=white,line width=0.35pt},line cap=round,line join=round,forget plot] coordinates {(4.761905,1.7995) (10.000000,1.3441) (20.000000,1.1630) (30.000000,1.1460)};

\end{groupplot}
\node[anchor=north,text=plotInk,font=\fontfamily{phv}\selectfont\footnotesize] at ($(group c1r1.south west)!0.5!(group c3r1.south east)+(0,-19pt)$) {Predicted values retained (\%)};
\end{tikzpicture}
    \caption{Budget--cost curves for $K=4$: the mean predictor at two
    resolutions and a causal GRU. Lines connect tested budgets; axis ranges
    differ. The GRU's Full endpoint is omitted. Complete results are in
    \Cref{app:additional-results}.}
    \label{fig:a100-budget}
\end{figure}
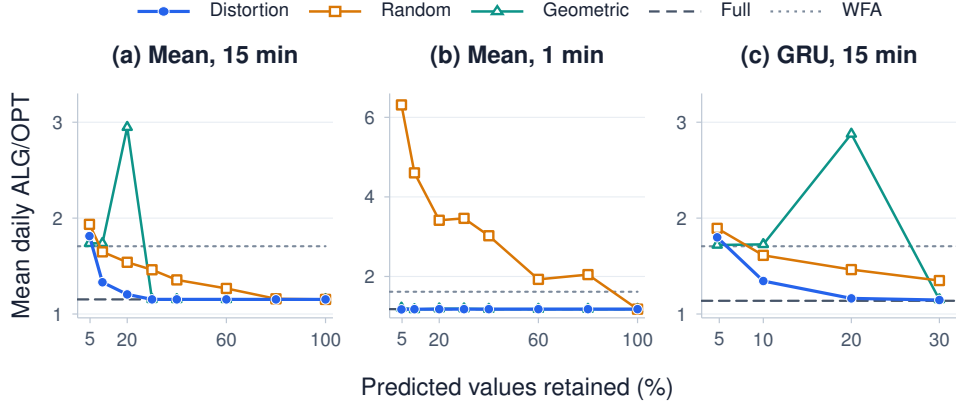

\paragraph{Compression performance.}
Retaining approximately 20\% of predicted dual values, \emph{Distortion}
achieves lower mean daily ALG/OPT than Random, Geometric, WFA, and DC in
all eight settings (\Cref{tab:a100-main}). At one minute, its mean ratio
differs from Full by less than $0.002$ for every tested server count,
preserving nearly the same observed performance with one fifth of the
values. At 15 minutes, the effect depends more strongly on the setting
and budget. For $K=4$, increasing the prediction fraction from 20\% to 30\%
reduces the ratio from $1.2043$ to $1.1527$, close to Full's $1.1524$
(\Cref{fig:a100-budget}(a)). Compression can also improve observed
performance: for $K=5$, the 20\% ratio is $1.2737$ compared with Full's
$1.3810$. Since Full uses imperfect predictions, reconstruction can
change its decisions beneficially; retaining more predicted values
therefore need not monotonically reduce cost
(\Cref{app:increment-resolution,app:increment-oracle-results}).

\begin{table}[htbp]
\centering\small
\setlength{\tabcolsep}{4pt}
\caption{Mean daily ALG/OPT on 365 dates with the mean predictor and approximately 20\% of predicted values retained. Bold marks row minima, including ties, without implying statistical significance. Additional comparisons are in \Cref{app:increment-resolution}.}\label{tab:a100-main}
\begin{tabular}{@{}lrrrrrrr@{}}
\toprule
Setting & $m/n$ & Full & Distortion & Random & Geom. & WFA & DC \\
\midrule
$K=3$, 1 min & 24/120 & 1.2712 & \textbf{1.2695} & 2.2326 & 1.6518 & 1.5286 & 2.5541 \\
$K=3$, 15 min & 24/120 & \textbf{1.1657} & 1.1659 & 1.5988 & 1.3081 & 1.4035 & 2.1838 \\
$K=4$, 1 min & 42/210 & 1.1788 & \textbf{1.1786} & 3.4129 & 1.1895 & 1.6160 & 3.5168 \\
$K=4$, 15 min & 42/210 & \textbf{1.1524} & 1.2043 & 1.5396 & 2.9478 & 1.7067 & 2.4071 \\
$K=5$, 1 min & 50/252 & 1.0502 & \textbf{1.0499} & 3.1009 & 1.0706 & 1.7300 & 4.9765 \\
$K=5$, 15 min & 50/252 & 1.3810 & \textbf{1.2737} & 2.2557 & 1.5288 & 1.9049 & 2.7607 \\
$K=6$, 1 min & 42/210 & \textbf{1.0548} & \textbf{1.0548} & 1.3530 & 1.0764 & 1.6367 & 5.8050 \\
$K=6$, 15 min & 42/210 & \textbf{1.3252} & 1.3349 & 1.6344 & 1.4776 & 1.9908 & 2.9673 \\
\bottomrule
\end{tabular}
\end{table}

\paragraph{Exact-value diagnostic.}
To examine compression with prediction error removed, we freeze the
training-selected landmarks and replace their predictions with exact
evaluation-day dual values, keeping the reconstruction and decision rule
unchanged. This diagnostic uses future information. Oracle Full achieves
OPT in every setting, whereas compressed exact values can still incur
excess cost. With approximately 20\% of values retained, Oracle Distortion's
mean ratios range from $1.0005$ to $1.0017$ across the four one-minute settings.
At 15 minutes, the ratio reaches $1.0854$ for $K=6$, compared with $1.0014$
at one minute. Thus landmark reconstruction can introduce decision loss
even when its inputs are exact, and this loss varies across the evaluated
settings. All recorded per-date theoretical cost bounds hold, although
they are generally loose (\Cref{app:increment-oracle-results}).

\paragraph{Compression across predictors.}
We further evaluate the same landmark sets with static, mean,
calendar-MLP, and causal-GRU predictors for $K=4$ at 15 minutes. At a 20\%
prediction fraction, Distortion achieves lower mean daily ALG/OPT than Random
and Geometric for all four predictors, extending the selection advantage beyond the
default mean predictor. For the GRU, retaining $42$ of $210$ predicted
dual values yields a ratio of $1.1630$, compared with $1.4633$ for Random,
$2.8774$ for Geometric, and $1.1378$ for Full. Increasing the prediction
fraction to 30\% brings Distortion to $1.1460$ (\Cref{fig:a100-budget}(c)). These
results show that the same offline landmark sets support effective
compression across the tested predictor families, including a recurrent
predictor that uses observed request history
(\Cref{app:a100-temporal-results}).

Complementing the main experiments, we directly evaluate the learning
procedure in \Cref{sec:learning} on synthetic instances by minimizing the
joint empirical objective over landmark sets and fixed prediction tables.
Results and checks of the theoretical cost bounds appear in
\Cref{app:synthetic-summary}.

\section{Related Work}
\label{sec:related-work}

\paragraph{Online MTS and Related Online Problems.}
MTS was introduced by \citet{borodin1992MTS}, who established the tight
deterministic competitive ratio of \(2n-1\). For randomized algorithms,
\citet{coester2022upper} give an \(O(\log^2 n)\) upper bound for every
\(n\)-point metric, with a matching worst-case lower bound due to
\citet{bubeck2023lower}.
Learning-augmented algorithms use predictions of events or actions to improve
online decisions \citep{mitzenmacher2022predictions}.
For arbitrary MTS, \citet{antoniadis2023MTSonline} combine action predictions
with classical competitive algorithms to achieve consistency, smoothness,
and robustness, with improved guarantees for caching and online matching
on the line. Related approaches use predictions of request arrival times
and locations for online TSP and dial-a-ride \citep{hu2026TSPonline},
next-arrival times for caching \citep{wei2020cachingonline}, and
distributions over the unknown horizon for ski rental \citep{kang2026skionline}.

\paragraph{Primal--Dual Framework and Predictions.}
\citet{bamas2020primaldual} introduce a prediction-guided primal--dual
framework for online covering, extended to convex covering by
\citet{grigorescu2025primaldual}. These methods seek consistency with
accurate predictions while preserving robustness.
Further developments address linear and semidefinite covering programs
\citep{grigorescu2022linear}, bounded allocation and ad-auctions
\citep{kevi2023bounded}, nonlinear packing \citep{thang2021packing}, and
nonlinear covering \citep{kevi2023covering}.
\citet{coester2026dual} predict optimal LP dual variables to guide online
decisions. For MTS, they interpret the dual as a time-reversed work function
and use an \(A^*\)-style rule, obtaining
\(\left(1+\frac{\eta}{\mathrm{OPT}}\right)\)-competitiveness, where
\(\eta\) is the sum over rounds of the spans of the Bellman residuals
of the predicted values.
We use the same value-greedy rule with values reconstructed from landmark
predictions.

\paragraph{Compressed Predictions.}
Prior work studies prediction budgets through fewer oracle calls, fewer bits
per prediction, and predictions for selected variables.
\citet{im2022parsimonious} reduce the number of next-arrival-time predictions
used for caching. For MTS, \citet{sadek2024} query the oracle only every
\(a\) steps. \citet{shin2026} use well-separated queries and virtual
predictions for online metric matching.
\citet{antoniadis2023} provide one bit per request for paging, using discard
or phase predictions. In offline approximation algorithms for dense
NP-hard problems, \citet{bampis2024} use one-bit predictions for a logarithmic
number of sampled variables.
\citet{daneshvaramoli2025} use a single value or interval estimating the
critical value in online knapsack.
We study predictions of canonical dual values at a fixed set of landmarks
while retaining the full action space. Our excess-cost bounds depend on the
landmark geometry and the accuracy of the landmark predictions.

\section{Conclusion}

We showed how landmark-compressed value predictions reduce prediction
dimension in MTS while controlling excess cost. Our bounds link this
tradeoff to metric coverage and prediction error, with matching worst-case
dependence on the covering radius for fixed-landmark access in the sparse
finite-line regime. On two states with one landmark, the Raw/Relative
separation shows that retaining absolute cost information can change
worst-case expected excess from linear in $T$ to horizon-independent.
Beyond coverage, reconstruction distortion captures
how well landmarks represent episode values and guides the joint learning
of landmarks and fixed prediction tables. Experiments on a Citi Bike-derived
$K$-server benchmark show that distortion-based selection can preserve much
of Full's observed performance using a small fraction of predicted values.
Natural next steps include adaptive landmark queries and learning
guarantees for contextual predictors.

\subsection*{AI use statement}

We used generative AI tools to assist with the derivation of the
learning-theoretic results, code implementation, data processing, and language
editing of the manuscript. The authors manually reviewed the AI-assisted work
and take responsibility for the final content of the paper, including its
theoretical claims, experimental results, and accompanying artifacts.

\subsection*{Ethics statement}

This work studies online algorithms using synthetic instances and publicly
available Citi Bike data. We do not identify any specific ethical concerns
arising from this research.

\subsection*{Reproducibility statement}

Proofs are provided in \Crefrange{app:preliminaries}{app:learning-proofs},
and experimental protocols in \Cref{app:experiment-details}.
The accompanying \path{mts-landmark-release.zip} contains source code,
experiment configurations, environment specifications, and archived reference
results. Its \path{README.md} and \path{docs/REPRODUCING.md} describe setup,
data preparation, and reproduction.



\bibliography{references}
\bibliographystyle{iclr2027_conference}

\clearpage
\appendix
\crefalias{section}{appendix}
\crefalias{subsection}{appendix}

\section{Preliminaries and Conventions}
\label{app:preliminaries}

This appendix supplies the Bellman and metric facts used in
\Cref{sec:setup,sec:upper_bound,sec:lower-bound}. Throughout,
$\tau=T-1$, $D=\max_{x,y\in M}d(x,y)$, and all service costs are finite
and nonnegative. A value $w_t$ is the optimal remaining cost \emph{after}
task $t$, so $w_T=0$ and $w_0(s_0)=\mathrm{OPT}$.
Landmarks are nonempty and chosen independently of the test episode.

\subsection{Canonical values and Bellman operators}
\label{app:bellman-properties}

\begin{lemma}[Bellman Lipschitzness and interval preservation]
\label{lem:bellman-properties}
For every finite function $f:M\to\mathbb R$, $B_cf$ is $1$-Lipschitz.
The operator is monotone and translation-equivariant. In particular, if
$\alpha\le f-g\le\beta$ pointwise, then
\begin{equation}
    \alpha\le B_cf-B_cg\le\beta,
    \qquad
    \|B_cf-B_cg\|_{\mathrm{sp}}
    \le\|f-g\|_{\mathrm{sp}}.
    \label{eq:bellman-interval}
\end{equation}
Every canonical value in \eqref{eq:bellman_recursion} is $1$-Lipschitz.
\end{lemma}
\begin{proof}
If $b$ minimizes the expression defining $(B_cf)(y)$, the triangle
inequality gives
\[
    (B_cf)(x)\le d(x,b)+c(b)+f(b)
    \le d(x,y)+(B_cf)(y).
\]
Interchanging $x,y$ proves Lipschitzness. Minimizing pointwise
inequalities proves monotonicity, and pulling a constant out of a minimum
gives $B_c(f+a)=B_cf+a$. Apply these properties to
$g+\alpha\le f\le g+\beta$. The final assertion follows from the Bellman
recursion and $w_T=0$.
\end{proof}

The dual interpretation is given in \Cref{sec:setup}. Our arguments use
the canonical Bellman values, rather than an arbitrary optimal dual solution.

\subsection{Span errors and Bellman gaps}
\label{app:bellman-gaps}

For a vector $z$ on a finite nonempty set, let $m_z=\min z$ and
$M_z=\max z$. Every constant $a$ satisfies
$\max|z-a|\ge(M_z-m_z)/2$, and equality holds at
$a=(M_z+m_z)/2$. This proves the identity in
\eqref{eq:landmark_prediction_error}. The span is invariant under
constant shifts and satisfies the triangle inequality.

\begin{lemma}[Pathwise gap identity and value-span bound]
\label{lem:gap-identity}
For any path starting at $s_0$, define
\begin{equation}
    \Delta_t=d(s_{t-1},s_t)+c_t(s_t)+w_t(s_t)
                 -w_{t-1}(s_{t-1}).
    \label{eq:gap-definition-appendix}
\end{equation}
Then $\Delta_t\ge0$ and
\begin{equation}
    \mathrm{ALG}-\mathrm{OPT}=\sum_{t=1}^T\Delta_t.
    \label{eq:gap-telescoping}
\end{equation}
If the path minimizes the predicted Bellman objective at every time and
$\widehat w_T=0$, then
\begin{equation}
    \mathrm{ALG}-\mathrm{OPT}
    \le\sum_{t=1}^{\tau}\|\widehat w_t-w_t\|_{\mathrm{sp}}.
    \label{eq:value-span-bound-appendix}
\end{equation}
\end{lemma}
\begin{proof}
Nonnegativity follows from Bellman optimality. In the sum of
\eqref{eq:gap-definition-appendix}, the value terms telescope to
$w_T(s_T)-w_0(s_0)=-\mathrm{OPT}$. For the second assertion, let
$x_t^\star$ be a Bellman minimizer from the algorithm's actual state
$s_{t-1}$, and set $e_t=\widehat w_t-w_t$. Comparing the chosen action
to $x_t^\star$ in the predicted objective yields
$\Delta_t\le e_t(x_t^\star)-e_t(s_t)\le\|e_t\|_{\mathrm{sp}}$.
At $t=T$, the prediction is exact, and hence $\Delta_T=0$.
\end{proof}

The terminal prediction is the zero function on \emph{all} of $M$.
It is not $E_L0=d(\cdot,L)$, which is generally positive away from
landmarks. No prediction at time zero is required by the greedy policy.

\subsection{Covering radii and prediction dimension}
\label{app:covering-geometry}

\begin{lemma}[Exact covering radius on a finite line]
\label{lem:line-radius}
For $M=\{0,\ldots,n-1\}$ with unit-spaced distance and $1\le m\le n$,
\begin{equation}
    r_m(M)=\left\lceil\frac{n-m}{2m}\right\rceil.
    \label{eq:line-radius-exact}
\end{equation}
\end{lemma}
\begin{proof}
All distances are integers, so a radius-$r$ ball covers at most $2r+1$
states. A cover with at most $m$ balls therefore requires
$m(2r+1)\ge n$. Conversely, if this inequality holds, split the line
into at most $m$ consecutive blocks of length at most $2r+1$ and choose
a middle state in each block. Each block has radius at most $r$.
\end{proof}

If $\mathcal N_M(r)=\min\{|L|:r(L)\le r\}$, then
\Cref{thm:global_upper_bound} gives excess at most $2\tau r$ with
$\mathcal N_M(r)$ exact landmark values per predicted round.
Translation equivariance of $E_L$ means that only $m-1$ anchored
relative values per round, or $\tau(m-1)$ over an episode, are
needed by the greedy policy. This counts scalar values;
landmark identities and numerical precision are separate resources.
In particular, it is not a bound for arbitrary information encoded into
real numbers. Formula~\eqref{eq:line-radius-exact} also explains why
$r_m$ cannot be replaced by a uniform $\Theta(n/m)$ statement near
$m=n$, where $r_n=0$.

\section{Upper-Bound Proofs and Refined Certificates}
\label{app:upper-bound-details}

We prove the claims in \Cref{sec:upper_bound} and the value-dependent
bound used by the learning objective in \Cref{sec:learning}.

\subsection{Reconstruction and the global guarantee}
\label{app:reconstruction-proof}

\begin{proof}[Proof of \Cref{lem:reconstruction}]
The envelope $E_Lv$ is monotone, translation-equivariant, and
$1$-Lipschitz, by the same minimum-and-triangle-inequality argument as
in \Cref{lem:bellman-properties}. If $v$ is $1$-Lipschitz on $L$, then
$(E_Lv)(\ell)=v(\ell)$ for every landmark: the term indexed by $\ell$
gives one inequality, and Lipschitzness gives the other. Arbitrary noisy
vectors need not be interpolated at their own landmarks.

For the canonical target, $w_t(\ell)+d(x,\ell)\ge w_t(x)$ for all
$\ell$. Choosing a nearest landmark $\ell_x$ also gives
\[
    E_L(w_t|_L)(x)
    \le w_t(\ell_x)+d(x,\ell_x)
    \le w_t(x)+2d(x,L).
\]
This proves \eqref{eq:exact_reconstruction}. Put
$a_t=\min_{\ell\in L}(v_t(\ell)-w_t(\ell))$ and subtract $a_t$
from every prediction. The landmark errors then lie in
$[0,2\delta_t]$, so monotonicity gives the pointwise bound
\begin{equation}
    0\le E_L(v_t-a_t)(x)-w_t(x)
          \le2d(x,L)+2\delta_t.
    \label{eq:normalized-pointwise-error}
\end{equation}
Taking the span proves \eqref{eq:reconstruction_span}. The subtraction
is only for analysis and does not require the algorithm to know $w_t$.
\end{proof}

\begin{proof}[Proof of \Cref{thm:global_upper_bound}]
Apply \Cref{lem:gap-identity} to the predicted functions
$\widehat w_t=E_Lv_t$ for $t<T$ and $\widehat w_T=0$. Summing
\eqref{eq:reconstruction_span} over the $\tau$ predicted rounds gives
\eqref{eq:main_upper_bound}. The comparison uses optimal successors
from the algorithm's own states, so it holds for every realized episode
and prediction sequence.
\end{proof}

Both $E_Lv_t$ and $w_t$ are $1$-Lipschitz. Their difference has span
at most $2D$, even if the landmark predictions are inconsistent or
arbitrarily large. Each summand of the global bound can consequently
be replaced by $\min\{2D,2r(L)+2\delta_t\}$. This is an additive
diameter bound; it does not imply a bounded competitive ratio when
$\mathrm{OPT}$ is small.

\subsection{A local successor-radius guarantee}
\label{app:local-upper-bound}

Define the true Bellman-optimal successor set and its distance to $L$ by
\begin{equation}
\begin{split}
    A_t^\star(a)&=\arg\min_{x\in M}
        \{d(a,x)+c_t(x)+w_t(x)\},\\
    \rho_t^\star(a;L)&=\min_{x\in A_t^\star(a)}d(x,L).
\end{split}
\label{eq:local-successor-radius}
\end{equation}

\begin{theorem}[Local successor-radius bound]
\label{thm:local-upper-bound}
The policy in \eqref{eq:landmark_greedy} satisfies
\begin{equation}
    \mathrm{ALG}-\mathrm{OPT}
    \le2\sum_{t=1}^{\tau}
          \bigl[\rho_t^\star(s_{t-1};L)+\delta_t\bigr].
    \label{eq:local-upper-bound}
\end{equation}
\end{theorem}
\begin{proof}
Normalize errors as in \eqref{eq:normalized-pointwise-error} and write
$e'_t=E_L(v_t-a_t)-w_t\ge0$. Choose $x_t^\star\in A_t^\star(s_{t-1})$
with smallest distance to $L$. The action comparison in
\Cref{lem:gap-identity} gives
\[
    \Delta_t\le e'_t(x_t^\star)-e'_t(s_t)
    \le e'_t(x_t^\star)
    \le2\rho_t^\star(s_{t-1};L)+2\delta_t.
\]
Sum over $t<T$ and use $\Delta_T=0$.
\end{proof}

Here $s_{t-1}$ depends on the candidate landmark set and table.
Clustering states on a fixed offline optimal path does not by itself
control this certificate. Validation in \Cref{app:validation-proof}
instead evaluates each candidate's own trajectory.

\subsection{Value-dependent reconstruction distortion}
\label{app:distortion-upper-bound}

For an episode $X$, define
\begin{equation}
    \kappa_t(L;X)=\max_{x\in M}\min_{\ell\in L}
        \{d(x,\ell)+w_t^X(\ell)-w_t^X(x)\}.
    \label{eq:distortion-definition}
\end{equation}
The exact reconstruction error is nonnegative and vanishes at every
landmark. Its minimum is therefore zero, and
\begin{equation}
    \kappa_t(L;X)
       =\|E_L(w_t^X|_L)-w_t^X\|_{\mathrm{sp}}
       \le2r(L).
    \label{eq:distortion-span-identity}
\end{equation}

\begin{theorem}[Distortion-plus-prediction bound]
\label{thm:distortion-upper-bound}
For every episode and arbitrary landmark predictions,
\begin{equation}
    \mathrm{ALG}-\mathrm{OPT}
       \le\sum_{t=1}^{\tau}
          \left[\kappa_t(L;X)+\|v_t-w_t^X|_L\|_{\mathrm{sp}}\right].
    \label{eq:distortion-upper-bound}
\end{equation}
\end{theorem}
\begin{proof}
Let $a_t$ and $b_t$ be the minimum and maximum landmark prediction
errors. Monotonicity and translation equivariance imply
\[
    a_t\le E_Lv_t-E_L(w_t^X|_L)\le b_t.
\]
The span of this difference is at most $b_t-a_t=2\delta_t$.
The span triangle inequality and \eqref{eq:distortion-span-identity}
give $\|E_Lv_t-w_t^X\|_{\mathrm{sp}}\le\kappa_t(L;X)+2\delta_t$.
Apply \Cref{lem:gap-identity}.
\end{proof}

Unlike \eqref{eq:local-upper-bound}, $\kappa_t$ is independent of the
online trajectory. This is why it can be used in
\eqref{eq:joint_learning_objective}. It can be much smaller than
$2r(L)$: if $w_t(x)=d(x,a)$ and $L=\{a\}$, the reconstruction is
exact and $\kappa_t=0$, even when $r(L)=D$.
All three certificates can also be combined after an episode:
\begin{equation}
    \mathrm{ALG}-\mathrm{OPT}\le
    \sum_{t=1}^{\tau}\min\left\{
       2D,\ \kappa_t(L;X)+2\delta_t,
       \ 2\rho_t^\star(s_{t-1};L)+2\delta_t\right\}.
    \label{eq:combined-certificate}
\end{equation}
Computing these certificates requires offline values; taking online
actions does not.

\subsection{Comparison with Bellman-residual analysis}
\label{app:residual-comparison}

To distinguish the direct analysis from a residual bound, let
$\widehat w_0$ be an auxiliary function and $\widehat w_T=0$, and set
\[
    h_t=B_{c_t}\widehat w_t-\widehat w_{t-1},
    \qquad \eta_t=\|h_t\|_{\mathrm{sp}}.
\]
On the greedy trajectory, the predicted Bellman equality gives
$\mathrm{ALG}=\widehat w_0(s_0)+\sum_t h_t(s_{t-1})$.
Along an offline optimal trajectory, the corresponding inequality gives
$\mathrm{OPT}\ge\widehat w_0(s_0)+\sum_t\min_x h_t(x)$.
Thus $\mathrm{ALG}-\mathrm{OPT}\le\sum_t\eta_t$.
Writing $e_t=\widehat w_t-w_t$, interval preservation yields
\begin{equation}
    \eta_t\le\|e_t\|_{\mathrm{sp}}+\|e_{t-1}\|_{\mathrm{sp}}.
    \label{eq:residual-adjacent-errors}
\end{equation}
For two adjacent exact envelopes, $0\le e_t,e_{t-1}\le2r(L)$,
and hence $\eta_t\le4r(L)$. Setting $\widehat w_0=w_0$ for analysis
and using $e_T=0$ yields the total estimate
$4\tau r(L)+4\sum_{t<T}\delta_t$. The direct proof of
\Cref{thm:global_upper_bound} pays each current value error once.

\begin{proposition}[Tightness of the one-step residual estimate]
\label{prop:residual-tightness}
The constant $4$ in the exact-envelope bound $\eta_t\le4r(L)$
cannot be improved in general, even for canonical future values.
\end{proposition}
\begin{proof}
On the four-state line $M=\{0,1,2,3\}$ take $L=\{0,3\}$, so
$r(L)=1$, and let
\[
\begin{aligned}
    w_t&=(0,1,0,1), & c_t&=(2,0,0,2),\\
    w_{t-1}=B_{c_t}w_t&=(2,1,0,1),&
    E_L(w_t|_L)&=(0,1,2,1).
\end{aligned}
\]
The preceding envelope is $(2,3,2,1)$, whereas
$B_{c_t}E_L(w_t|_L)=(2,1,2,3)$. Their difference is
$(0,-2,0,2)$, of span $4$. The nonnegative $1$-Lipschitz function
$w_t$ is realized by a final task $c_{t+1}=w_t$, since
$B_{w_t}0=w_t$. Thus these are canonical values of a legal episode.
\end{proof}

This example concerns the residual estimate. It does not prove that the
coefficient $2$ of every algorithmic guarantee is optimal. The matching
order lower bound uses the assumptions on fixed landmarks in
\Cref{thm:lower-bound}.

\section{Full Proof of the Lower Bound for Fixed Landmarks}
\label{app:lower-bound-details}

This appendix completes the construction in \Cref{sec:lower-bound},
including finite reset penalties, whole-table indistinguishability,
and the quantifiers for randomized landmark selection.

\subsection{The gadget and legal service costs}
\label{app:lower-bound-gadget}

As in \Cref{sec:lower-bound}, let $D=n-1$,
$h=\lfloor D/(8m)\rfloor$, and $o=s_0$.
Choose the baseline $C=D+h$ and the finite penalty $H=4D$.
For $i=0,\ldots,2m-1$, set
\[
    p_i=4hi,\quad I_i=(p_i,p_i+4h),\quad
    a_i=p_i+h,\quad z_i=p_i+2h,\quad b_i=p_i+3h.
\]
Intervals are intersected with the discrete state space when used in set
membership. They are disjoint, and all their endpoints lie in $M$ because
$8mh\le D$. The assumption $D\ge16m$ ensures $h\ge2$.

For one interval, suppress $i$ and define
\[
    \phi(x)=(h-|x-a|)_+-(h-|x-b|)_+,
    \qquad f_\sigma(x)=C-\sigma\phi(x),\quad \sigma\in\{-1,1\}.
\]
The two triangular supports have disjoint interiors and slopes of
magnitude at most one. Hence $f_\sigma$ is $1$-Lipschitz,
$f_\sigma=C$ outside $I$, and
$\{f_\sigma(a),f_\sigma(b)\}=\{C-h,C+h\}$.
The four stages described in \Cref{sec:lower-bound} are realized by
the following enter, decide, reveal, and reset tasks:
\[
\begin{aligned}
c^{\mathrm{in}}(x)&=\begin{cases}0,&x=z,\\ H,&x\neq z,\end{cases}
&
c^{\mathrm{dec}}(x)&=\begin{cases}0,&x\notin I\text{ or }x\in\{a,b\},\\ H,&\text{otherwise},\end{cases}\\
c_\sigma^{\mathrm{rev}}(x)&=f_\sigma(x)-d(x,o),
&
c^{\mathrm{out}}(x)&=\begin{cases}0,&x=o,\\ H,&x\neq o.\end{cases}
\end{aligned}
\]
In particular,
$c_\sigma^{\rm rev}(x)=f_\sigma(x)-d(x,o)\ge D-d(x,o)\ge0$.
Thus all four tasks have finite, nonnegative service costs.

\paragraph{Why a fixed finite penalty forces a reset.}
For any $1$-Lipschitz continuation $V$, let a reset task cost zero at
$q$ and $H$ elsewhere. From any incoming state $x$, the score of a
non-target state $y$ minus that of $q$ is at least
\begin{equation}
    H+d(x,y)-d(x,q)+V(y)-V(q)\ge H-2D>0.
    \label{eq:reset-gap}
\end{equation}
The reset is therefore uniquely Bellman-optimal. The bound does not
depend on the number of later gadgets or their accumulated cost.

\subsection{Canonical values throughout the concatenation}
\label{app:lower-bound-recurrence}

Let $g_\sigma=B_{c^{\rm dec}}f_\sigma$. At $z$, the favorable
candidate costs $h+C-h=C$; the unfavorable one costs $C+2h$;
an exterior state costs at least $C+2h$; and any other interior state
costs at least $H+C-h>C$. Thus $g_\sigma(z)=C$.
Nonnegative service costs and Lipschitzness imply
$g_\sigma\ge B_0f_\sigma=f_\sigma$. At an exterior state, staying
costs exactly $C$, so
\begin{equation}
    g_\sigma(z)=C,\qquad g_\sigma(x)=C\quad(x\notin I).
    \label{eq:gadget-exterior-values}
\end{equation}

Concatenate $N=\lfloor T/4\rfloor$ gadgets with interval indices
$i_1,\ldots,i_N$ and signs $\sigma_1,\ldots,\sigma_N$, and append
$T-4N$ zero tasks. Define
\begin{equation}
    K_j=\sum_{q=j+1}^{N}\bigl[d(o,z_{i_q})+C\bigr].
    \label{eq:gadget-suffix-cost}
\end{equation}
These constants depend on interval choices but on no hidden sign.
Backward induction gives
\begin{equation}
\begin{aligned}
    w_{4j-4}(x)&=d(x,z_{i_j})+C+K_j,\\
    w_{4j-3}(x)&=g_{\sigma_j}(x)+K_j,\\
    w_{4j-2}(x)&=f_{\sigma_j}(x)+K_j,\\
    w_{4j-1}(x)&=d(x,o)+K_j.
\end{aligned}
\label{eq:gadget-full-recurrence}
\end{equation}
Here $f_{\sigma_j},g_{\sigma_j}$ refer to interval $i_j$.
To see the induction, the zero-task suffix has value zero. At a return
task, \eqref{eq:reset-gap} forces $o$, whose subsequent value is $K_j$.
The reveal task then gives
\[
    \min_y\{d(x,y)+f_{\sigma_j}(y)-d(y,o)+d(y,o)+K_j\}
       =f_{\sigma_j}(x)+K_j,
\]
using $B_0f_{\sigma_j}=f_{\sigma_j}$. The decide task gives
$g_{\sigma_j}+K_j$, and the enter task resets to $z_{i_j}$ with
continuation $C+K_j$ by \eqref{eq:gadget-exterior-values}.
Consequently,
\begin{equation}
    \mathrm{OPT}=\sum_{j=1}^{N}[d(o,z_{i_j})+C],
    \label{eq:gadget-optimum}
\end{equation}
also independently of every sign.

\subsection{Indistinguishability of the entire Raw table}
\label{app:whole-table-indistinguishability}

Fix $L$ and an unhit interval $L\cap I_{i_j}=\varnothing$.
At each landmark, $f_{\sigma_j}$ and $g_{\sigma_j}$ both equal $C$.
Equations~\eqref{eq:gadget-suffix-cost}--\eqref{eq:gadget-optimum}
therefore show that flipping $\sigma_j$ leaves \emph{every} revealed
dual value unchanged, at every time including time zero. In particular,
no earlier absolute value reveals the sign through its continuation cost.

The tasks through the decide stage of gadget $j$ do not depend on
$\sigma_j$ either. Conditional on the interval choices, the other
signs, and the algorithm's randomness, the two choices of $\sigma_j$
thus induce identical observations and identical actions until that
gadget's reveal stage. With a fair sign, it remains a fair hidden bit
at the decision round, even when the full Raw table was given in advance.

\subsection{Loss on an unhit interval}
\label{app:lower-bound-decision-loss}

Use the Bellman gaps in \Cref{lem:gap-identity}. If the enter action
is not $z$, then \eqref{eq:reset-gap} gives
$\Delta_{\rm in}\ge H-2D\ge h$. Otherwise the decision begins at $z$,
where the Bellman value is $C+K_j$. The favorable candidate has gap
zero, the unfavorable candidate has gap $2h$, an exterior state has
gap at least $2h$, and any other interior state has gap at least $H-h$.
Because the sign is hidden, either candidate has conditional expected
gap $h$, and every other action has at least this much. This also holds
for randomized mixtures of actions. Therefore
\begin{equation}
    \E[\Delta_{\rm in}+\Delta_{\rm dec}
             \mid L\cap I_{i_j}=\varnothing]\ge h.
    \label{eq:unhit-gadget-loss}
\end{equation}
Later gaps are nonnegative and cannot cancel this loss.

\subsection{Oblivious inputs and the final constant}
\label{app:lower-bound-averaging}

Choose every interval index independently and uniformly from
$\{0,\ldots,2m-1\}$, and choose all signs independently and uniformly.
This distribution fixes the full episode without observing any realized
algorithmic random bits. Any realized landmark set with $|L|\le m$
hits at most $m$ of the $2m$ disjoint intervals. As $L$ is chosen before
test information is observed,
\[
    \Pr[L\cap I_{i_j}=\varnothing\mid L]\ge\tfrac12.
\]
Combining this with \eqref{eq:unhit-gadget-loss} and summing gaps yields
$\E[\mathrm{ALG}-\mathrm{OPT}]\ge Nh/2$, with expectation over both
input and algorithm randomness. Averaging over the finite input support
fixes a deterministic episode whose expected excess over algorithm
randomness is at least $Nh/2$.

Finally, $D\ge16m$ implies
$h=\lfloor D/(8m)\rfloor\ge D/(16m)$. By
\Cref{lem:line-radius}, $r_m(M)\le D/m$ in this regime, so
$h\ge r_m(M)/16$. Therefore the selected episode satisfies
\[
    \E[\mathrm{ALG}]-\mathrm{OPT}
        \ge\frac1{32}\left\lfloor\frac T4\right\rfloor r_m(M),
\]
which proves \Cref{thm:lower-bound}. The zero-task padding leaves all
preceding canonical values unchanged and adds no negative gap.

This proof requires the landmark set to be fixed before the episode.
Adaptively querying the future value at $a_i$ and at a state outside
$I_i$ reveals a value difference
$-\sigma h$ at the decision stage. It therefore does not establish an
adaptive-query lower bound. The nearly-full Raw regime is addressed
separately in \Cref{app:information-interfaces}.

\section{Absolute and Relative Prediction Interfaces}
\label{app:information-interfaces}

This appendix proves the interface separation in
\Cref{thm:raw-relative-separation} and explains the sparse-regime
restriction in \Cref{thm:lower-bound}. The guarantees here allow
policies other than the value-greedy rule.

\subsection{A Raw budget policy with one unreported state value}
\label{app:raw-budget-policy}

\begin{theorem}[One unreported state value under Raw access]
\label{thm:raw-budget}
Suppose $M\setminus L=\{q\}$ and $L\ne\varnothing$. With round-by-round
access to exact absolute landmark values, a deterministic policy satisfies
\begin{equation}
    \mathrm{ALG}\le\mathrm{OPT}+2r(L)
    \label{eq:raw-budget-bound}
\end{equation}
on every episode, without being supplied $\mathrm{OPT}$.
\end{theorem}

\paragraph{Policy.}
For $T=1$, minimize the observed movement-plus-service cost directly.
For $T\ge2$, after receiving $c_1$ and $w_1|_L$, set
\begin{equation}
    U=\min_{x\in M}
       \{d(s_0,x)+c_1(x)+E_L(w_1|_L)(x)\}.
    \label{eq:raw-initial-budget}
\end{equation}
Let $P_{t-1}$ be the amount already paid and
$B_{t-1}=U-P_{t-1}$. At time $t$, select a landmark satisfying
\begin{equation}
    d(s_{t-1},\ell)+c_t(\ell)+w_t(\ell)\le B_{t-1},
    \label{eq:raw-feasible-landmark}
\end{equation}
if one exists, using a fixed tie rule. Otherwise select the unique
state $q$ outside $L$. At $t=T$, the required values are the known zeros.

\begin{proof}[Proof of \Cref{thm:raw-budget}]
Exact reconstruction gives
$w_1\le E_L(w_1|_L)\le w_1+2r(L)$, and hence
$\mathrm{OPT}\le U\le\mathrm{OPT}+2r(L)$. Maintain the invariant
\begin{equation}
    B_{t-1}\ge w_{t-1}(s_{t-1}).
    \label{eq:raw-budget-invariant}
\end{equation}
It holds initially. If a feasible landmark is chosen, subtracting its
immediate cost from \eqref{eq:raw-feasible-landmark} proves the next
instance of the invariant. If no landmark is feasible, Bellman
optimality and \eqref{eq:raw-budget-invariant} guarantee that some
state has movement-plus-service-plus-continuation cost at most
$B_{t-1}$. That state must be $q$, so choosing it preserves the
invariant without observing $w_t(q)$. At the end,
$U-\mathrm{ALG}=B_T\ge w_T(s_T)=0$, proving the bound.
\end{proof}

If $\mathrm{OPT}$ is additionally revealed, initializing $U=\mathrm{OPT}$
makes the policy optimal. This extra information is not assumed above;
it is available, for example, if the Raw interface also reveals $w_0(s_0)$.
The policy depends on a common absolute cost scale across times, so
independent additive shifts of successive revealed value vectors invalidate
its budget comparisons. This is compatible with shift invariance of
the different greedy policy in \Cref{sec:upper_bound}.

\subsection{Proof of the two-state separation}
\label{app:raw-relative-separation}

\begin{proof}[Proof of \Cref{thm:raw-relative-separation}]
The Raw upper bound is \Cref{thm:raw-budget}. With one landmark,
every revealed relative value is zero, regardless of the chosen landmark.
Label the initial state $0$, the other state $1$, and set
$H=4a$. For each independent gadget, choose a fair hidden bit
$b\in\{0,1\}$ and use the following tasks:
\[
\begin{array}{c|cc}
    \text{task}&c(0)&c(1)\\ \hline
    \text{enter}&0&H\\
    \text{decide}&a&0\\
    \text{reveal}&H\,\mathbf{1}\{b=1\}&H\,\mathbf{1}\{b=0\}\\
    \text{return}&0&H
\end{array}
\]
Every reset is Bellman-optimal by \eqref{eq:reset-gap}, with $D=a$.
Failure to enter $0$ creates gap at least $H-2a=2a$. Otherwise, both
decision actions have immediate cost $a$. Writing $K$ for the
continuation cost after returning to $0$, the future value after the
decision at state $x$ is
$d(x,b)+d(b,0)+K$. The Bellman-optimal decision is $b$, and choosing
the other state has gap exactly $a$. The current task, all past tasks,
and the identically zero relative values reveal nothing about the fresh bit.
Thus the conditional expected decision gap is at least $a/2$.
This argument covers every randomized decision rule.

Repeat $\lfloor T/4\rfloor$ gadgets and append zero tasks. Sum the
nonnegative Bellman gaps and average over the hidden bits to fix an
oblivious input satisfying \eqref{eq:relative-lower-bound}.
\end{proof}

The Raw excess bound is independent of $T$, whereas the Relative interface
can force excess linear in $T$. Thus the landmark budget
alone does not characterize the information available to arbitrary
history-dependent policies.

\section{Learning Guarantees and Optimization Details}
\label{app:learning-proofs}

We supply proofs for \Cref{sec:learning}. Training samples are ${N_{\mathrm{train}}}$
independent complete episodes from $\mathcal D$; no independence among
the tasks of a single episode is assumed. Every fitted table is fixed
across test episodes. Set $\tau=T-1$ and choose an anchor $\ell_0\in L$
by a deterministic rule.

\subsection{Envelope closure and a compact class}
\label{app:envelope-closure}

\begin{lemma}[Closure preserves decisions and contracts target error]
\label{lem:envelope-closure}
For any vector $v$ on $L$, let $\bar v=(E_Lv)|_L$. Then $\bar v$
is $1$-Lipschitz, $E_L\bar v=E_Lv$, and, for every $1$-Lipschitz
target $w$,
\begin{equation}
    \|\bar v-w|_L\|_{\mathrm{sp}}
       \le\|v-w|_L\|_{\mathrm{sp}}.
    \label{eq:closure-span-contraction}
\end{equation}
\end{lemma}
\begin{proof}
The envelope is $1$-Lipschitz, and its restriction is therefore
$1$-Lipschitz. Expanding the repeated envelope gives
\[
    E_L\bar v(x)=\min_{\ell,\ell'\in L}
         \{v(\ell')+d(\ell,\ell')+d(x,\ell)\}.
\]
The triangle inequality bounds this below by $E_Lv(x)$. Taking
$\ell=\ell'$ gives the reverse inequality. If
$a\le v-w|_L\le b$, monotonicity and translation equivariance yield
$E_L(w|_L)+a\le E_Lv\le E_L(w|_L)+b$.
The exact envelope agrees with $w$ on $L$, so restricting this
inequality to $L$ proves \eqref{eq:closure-span-contraction}.
\end{proof}

After closure, subtracting $\bar v(\ell_0)$ only translates the
extended function. Thus there is no loss of greedy decision rules
or increase of span loss in restricting to
\begin{equation}
    \mathcal V_L=\left\{v:
        v_t(\ell_0)=0,\quad
        |v_t(\ell)-v_t(\ell')|\le d(\ell,\ell')
        \ \text{for all }t,\ell,\ell'\right\}.
    \label{eq:compact-table-class}
\end{equation}
This is a closed bounded polytope with $\tau(m-1)$ free parameters,
each in $[-D,D]$. It is nonempty, and continuous empirical risks
attain their minima on it. The anchored true labels
$q_t^X(\ell)=w_t^X(\ell)-w_t^X(\ell_0)$ are also $1$-Lipschitz.
For any $u,v$ on $L$,
\begin{equation}
    \left|\|u-q_t^X\|_{\mathrm{sp}}
             -\|v-q_t^X\|_{\mathrm{sp}}\right|
       \le2\|u-v\|_\infty.
    \label{eq:span-loss-lipschitz}
\end{equation}
If $v$ is feasible, its difference from $q_t^X$ has span at most
$2D$. These bounds hold independently of the absolute scale of
the service costs.

\subsection{Proof of the fixed-table PAC guarantee}
\label{app:fixed-table-proof}

\begin{proof}[Proof of \Cref{thm:fixed_table_learning}]
Fix $t$ and write $\ell_{t,v_t}(X)=\|v_t-q_t^X\|_{\mathrm{sp}}$.
Partition $[-D,D]^{m-1}$ into cells of side
$a=\varepsilon/16$, choosing a feasible representative in each
nonempty cell. The representatives form an $a$-net of size at most
\[
    N_\varepsilon\le(1+32D/\varepsilon)^{m-1}.
\]
For a fixed representative, the ${N_{\mathrm{train}}}$ losses are independent across
episodes and lie in $[0,2D]$. Hoeffding's inequality gives
\begin{equation}
    \Pr\left[\left|\frac1{N_{\mathrm{train}}}\sum_{i=1}^{N_{\mathrm{train}}}
       \ell_{t,v_t}(X^{(i)})-\E_X\ell_{t,v_t}(X)\right|
       >\frac\varepsilon4\right]
       \le2\exp\left(-\frac{{N_{\mathrm{train}}}\varepsilon^2}{32D^2}\right).
    \label{eq:fixed-table-hoeffding}
\end{equation}
A union bound over all representatives and all $\tau$ times has
failure probability at most $\beta$ under
\eqref{eq:learning_sample_size}. On the resulting event, replacing
an arbitrary feasible vector by its representative changes empirical
and population losses by at most $2a$ each, by
\eqref{eq:span-loss-lipschitz}. Therefore the deviation at each time
is at most $\varepsilon/4+4a=\varepsilon/2$ uniformly over its
feasible vectors. Averaging over time proves
\begin{equation}
    \sup_{v\in\mathcal V_L}|\widehat R_L(v)-R_L(v)|
        \le\varepsilon/2.
    \label{eq:fixed-table-uniform-convergence}
\end{equation}
For a population minimizer $v^\star$ and any empirical minimizer
$\widehat v$, this implies
\[
    R_L(\widehat v)
    \le\widehat R_L(\widehat v)+\varepsilon/2
    \le\widehat R_L(v^\star)+\varepsilon/2
    \le R_L(v^\star)+\varepsilon.
\]
Finally, $\|v_t-q_t^X\|_{\mathrm{sp}}=2\delta_t$ and
\Cref{thm:global_upper_bound} imply
\eqref{eq:learning_cost_bound} after taking expectation over a
fresh test episode.
\end{proof}

The expectation over the test episode is conditional on the learned
table; the high-probability event is over training episodes. The union
bound over times needs no within-episode independence. Because the
risk averages over $\tau$ times, a total statistical excess target
$\varepsilon_{\rm tot}$ requires setting
$\varepsilon=\varepsilon_{\rm tot}/\tau$, producing a quadratic
horizon factor in the sample bound. For $m=1$, the anchored class
contains only the zero table and its span risk is identically zero;
compression can still incur nonzero online excess.

\subsection{The empirical-risk linear program}
\label{app:learning-lp}

Let $q_{it\ell}=q_t^{X^{(i)}}(\ell)$. The LP used in
\Cref{sec:learning,sec:experiments} is
\begin{equation}
\begin{aligned}
    \min_{v,u,b}\quad&\frac1{{N_{\mathrm{train}}}\tau}
                      \sum_{i=1}^{N_{\mathrm{train}}}\sum_{t=1}^{\tau}(u_{it}-b_{it})\\
    \text{subject to}\quad
       &v_t(\ell)-q_{it\ell}\le u_{it} &&\text{for all }i,t,\ell,\\
       &b_{it}\le v_t(\ell)-q_{it\ell} &&\text{for all }i,t,\ell,\\
       &v_t(\ell)-v_t(\ell')\le d(\ell,\ell')
                                         &&\text{for all }t,\ell,\ell',\\
       &v_t(\ell_0)=0                    &&\text{for all }t.
\end{aligned}
\label{eq:span-erm-lp}
\end{equation}
The ordered-pair constraints include both directions of the Lipschitz
condition. At an optimum, $u_{it}$ is the largest landmark error and
$b_{it}$ the smallest, since otherwise either variable can be tightened
to decrease the objective. The objective is thus exactly
$\widehat R_L(v)$.

Before eliminating anchors, the LP has $\tau m+2{N_{\mathrm{train}}}\tau$ variables,
$2{N_{\mathrm{train}}}\tau m+\tau m(m-1)$ displayed inequalities, and $\tau$ anchor
equalities. Bounds $[-D,D]$ on non-anchor table entries are
redundant mathematically and are imposed in the implementation.
The variables $u,b$ are unrestricted real numbers. With rational data,
the labels and this LP have polynomial encoding size and the fixed-$L$
problem is polynomial-time solvable. Computing all training labels by
generic dynamic programming costs $O({N_{\mathrm{train}}}Tn^2)$.

Solver and storage conventions are given in \Cref{app:implementation-details}.

\subsection{Joint selection of landmarks and tables}
\label{app:joint-learning-proof}

Let $\widehat J$ be the empirical version of
\eqref{eq:joint_learning_objective}, with $\kappa_t$ given by
\eqref{eq:distortion-definition}.

\begin{theorem}[Joint landmark/table generalization]
\label{thm:joint-learning}
If
\begin{equation}
    {N_{\mathrm{train}}}\ge\frac{128D^2}{\varepsilon^2}
       \left[\log\binom nm
       +(m-1)\log\left(1+\frac{32D}{\varepsilon}\right)
       +\log\frac{2\tau}{\beta}\right],
    \label{eq:joint-sample-size}
\end{equation}
then, with probability at least $1-\beta$, every output
$(\widehat L,\widehat v)$ whose empirical objective is within
$\gamma\ge0$ of the global empirical minimum satisfies
\begin{equation}
    \E_X[\mathrm{ALG}_{\widehat L,\widehat v}-\mathrm{OPT}]
       \le\tau\left[
         \inf_{|L|=m,\,v\in\mathcal V_L}J(L,v)
            +\varepsilon+\gamma\right].
    \label{eq:joint-learning-bound}
\end{equation}
\end{theorem}
\begin{proof}
For each $L$, every timewise summand in $J$ lies in $[0,4D]$:
both $\kappa_t$ and the landmark error span are at most $2D$.
The distortion is independent of $v_t$, so the summand remains
$2$-Lipschitz in that vector. Use the same $\varepsilon/16$ net
as in \Cref{app:fixed-table-proof}, and union-bound over all
$\binom nm$ landmark sets as well as the $\tau$ times. The exponent
in \eqref{eq:fixed-table-hoeffding} becomes
$-{N_{\mathrm{train}}}\varepsilon^2/(128D^2)$. Under \eqref{eq:joint-sample-size},
$|J-\widehat J|\le\varepsilon/2$ uniformly. For a population
minimizer $(L^\star,v^\star)$, empirical suboptimality gives
\[
\begin{split}
    J(\widehat L,\widehat v)
    &\le\widehat J(\widehat L,\widehat v)+\varepsilon/2\\
    &\le\widehat J(L^\star,v^\star)+\gamma+\varepsilon/2
     \le J(L^\star,v^\star)+\gamma+\varepsilon.
\end{split}
\]
Use \Cref{thm:distortion-upper-bound} to convert risk to excess.
\end{proof}

For each $L$, minimizing $\widehat J$ uses the same LP because its
distortion term is constant in $v$. Enumerating all $\binom nm$ sets
therefore realizes global empirical minimization in exact arithmetic.
The combinatorial factor is not polynomial for unrestricted $m$.
The pilots enumerate all 36 pairs and solve their LPs numerically;
they do not provide exact-arithmetic optimality certificates. A general
heuristic requires an explicit optimization-error allowance $\gamma$
unless its global suboptimality is certified.

\subsection{Independent validation of frozen candidates}
\label{app:validation-proof}

Let $\mathcal A_1,\ldots,\mathcal A_{J_0}$ be value-greedy
policies generated and fitted using only training data, and then frozen.
For each candidate define its normalized episode excess
\[
    Z_j(X)=\frac{\mathrm{ALG}_j(X)-\mathrm{OPT}(X)}{\tau}.
\]
Select $\widehat j$ by minimizing mean $Z_j$ over independent
validation episodes.

\begin{theorem}[Finite-candidate validation]
\label{thm:validation-selection}
If
\begin{equation}
    N_{\mathrm{val}}\ge\frac{8D^2}{\varepsilon^2}
                        \log\frac{2J_0}{\beta},
    \label{eq:validation-sample-size}
\end{equation}
then, with probability at least $1-\beta$,
$\E_X Z_{\widehat j}(X)\le\min_{j\le J_0}\E_X Z_j(X)+\varepsilon$.
\end{theorem}
\begin{proof}
Condition on the training set and all randomness used to construct
the candidates. By \Cref{lem:gap-identity} and the Lipschitz diameter
bound, $0\le Z_j\le2D$ for every candidate and episode. Hoeffding's
inequality gives a failure probability at most
$2\exp(-N_{\mathrm{val}}\varepsilon^2/(8D^2))$ for deviation
$\varepsilon/2$ for one candidate. Union-bound over $J_0$ candidates
and compare the empirical winner to the population-best candidate,
using the deviation bound twice. The conclusion holds for every
conditioning and hence also without conditioning.
\end{proof}

The guarantee compares to the best member of the frozen family, not
the globally best landmark set. It permits arbitrary training-only
candidate generation and needs no global optimization certificate.
Repeatedly modifying candidates using validation outcomes would no
longer satisfy the stated frozen-family assumption. Test episodes are
used only for final evaluation. None of the learning results above
asserts a complexity bound for unrestricted contextual predictors.

\section{Experimental Protocol and Reproducibility}
\suppressfloats[t]
\label{app:experiment-details}

This appendix expands the protocol in \Cref{sec:experiments}, first for
the Citi Bike-derived benchmark and then for the synthetic fixed-table
pilots. Additional results appear in
\Cref{app:additional-results}.

\subsection{Demand data and preprocessing}
\label{app:citibike-data}

The benchmark uses trip start times and coordinates from the official
Citi Bike archive \citep{citibikeData}. The frozen input comprises the
2023 annual archive and twelve monthly archives for each of 2024 and 2025:
25 ZIP files. Across 1,096 unique dates,
125,181,021 raw trip records yield 100,861,883 records after filtering.
\Cref{tab:citibike-splits} specifies the chronological split. The final
refit includes the initial training and validation periods; it is not
an additional independent data block.

\begin{table}[htbp]
\centering\small
\caption{Citi Bike data split. The final refit overlaps the two preceding rows.}
\label{tab:citibike-splits}
\begin{tabular}{@{}llrr@{}}
\toprule
Split & Dates & Days & Filtered trips \\
\midrule
Initial fit & 2023-01-01--2024-09-30 & 639 & 55,497,434 \\
Validation & 2024-10-01--2024-12-31 & 92 & 8,910,606 \\
Final refit & 2023-01-01--2024-12-31 & 731 & 64,408,040 \\
Evaluation & 2025-01-01--2025-12-31 & 365 & 36,453,843 \\
\bottomrule
\end{tabular}
\end{table}

The geographic filter is the rectangle with latitude $[40.70,40.88]$ and
longitude $[-74.025,-73.90]$, not an exact borough boundary. The observed
latitude range in the initial 639 training days is divided into ten
equal-width bins and frozen. Later coordinates outside that fitted range
are clipped to the edge bins. Longitude is used only in the rectangle
filter; distances are measured in latitude-bin units. There is no road
network, inventory, capacity, return-trip, traffic, or truck-loading model.

Each slot requests the bin with the largest trip count, breaking count
ties toward the smallest index. Empty slots receive the special label
$-1$; missing date coverage raises an error. Fifteen-minute requests
are formed by summing minute/bin counts before taking the most frequent
bin, rather than aggregating minute-level modal labels. Archive-local
New York wall-clock times give 96 or 1,440 slots per day, with repeated
daylight-saving slots merged. Every day starts a new service episode.

Modal aggregation concentrates the requests: bins 1 and 2 account for
90.79\% of nonempty 15-minute requests and 75.72\% of nonempty one-minute
requests across the processed traces. Mean active bins per day are 3.89
and 8.09, respectively; within-day adjacent-repeat rates, including empty
labels, are 69.81\% and 40.44\%. The 2025 traces contain eight empty
15-minute slots and 3,639 empty minute slots. These statistics delimit the
benchmark's scope. The finer trace is less concentrated but can be easier
to compress, so concentration alone does not explain the results.

\subsection{Configuration metric, online rule, and MTS encoding}
\label{app:citibike-model}

Use $K$ for the number of servers and $m=|L|$ for the landmark budget,
consistently with the theory. The report's state count $N$ is denoted by
$n$ here; its landmark budget $m$ uses the same notation. A state is a sorted set of distinct bins:
\begin{equation}
    \mathcal C_K=\{C\subseteq\{0,\ldots,9\}:|C|=K\},\quad
    n=\binom{10}{K},\quad
    d(C,C')=\sum_{i=1}^{K}|c_i-c'_i|.
    \label{eq:citibike-metric}
\end{equation}
The MTS states are whole configurations in the induced matching metric.
The starting configurations for $K=3,4,5,6$ are respectively
$\{0,4,9\}$, $\{0,3,6,9\}$, $\{0,2,4,7,9\}$, and $\{0,2,4,5,7,9\}$;
the state counts are $120,210,252,210$.

Using the paper's one-based request convention, let $r_t$ be the request
at round $t$. The exact continuation values and the predicted-value rule are
\begin{align}
    w_T^d(C)&=0, &
    w_{t-1}^d(C)&=\min_{C':r_t^d\in C'}
            \{d(C,C')+w_t^d(C')\}, \label{eq:citibike-bellman}\\
    C_t&\in\argmin_{C':r_t\in C'}
            \{d(C_{t-1},C')+\widetilde P_t(C')\}. &&
    \label{eq:citibike-policy}
\end{align}
Here $d$ as a superscript denotes the date, and
$\mathrm{OPT}_d=w_0^d(C_0)$. For empty requests, the implementation uses
the identity Bellman operator and leaves the online configuration unchanged.
For nonempty requests, several servers may move and all movement is charged.
Scores within $10^{-12}$ of the minimum are tied; smaller movement is
preferred, followed by the fixed configuration index. Exact DP uses float64.
Oracle Full receives exact evaluation-day values at every state and must
attain OPT; fitted predictors never receive these values during decisions.

\paragraph{Finite service costs.}
The hard feasibility condition admits an equivalent finite-cost MTS
encoding for these value-greedy policies and the offline optimum. Let
$D_{\mathcal C}$ be the configuration diameter, choose any $H>2D_{\mathcal C}$,
and set $c_t(C)=0$ if $r_t\in C$ and $c_t(C)=H$ otherwise.
For any $1$-Lipschitz continuation $f$, any infeasible state $A$, and any
feasible state $B$,
\[
 d(C,B)+f(B)-[d(C,A)+H+f(A)]
 \le 2d(A,B)-H<0.
\]
An infeasible state therefore cannot minimize the finite-cost objective.
Backward induction from the terminal zero identifies its canonical values
with \eqref{eq:citibike-bellman}; all evaluated predictions and their
extensions are also $1$-Lipschitz. For an empty request, set $c_t\equiv0$:
the Bellman operator fixes a $1$-Lipschitz vector, and minimum-movement
tie-breaking selects the current state. The implementation directly enforces
feasibility, avoiding a large numerical penalty.

\paragraph{Classical references.}
WFA \citep{koutsoupias1995kserver} maintains the extended work function for the request history and the
fixed initial configuration, then minimizes its work value plus movement
over feasible next configurations. DC \citep{chrobak1991server} moves the closest server for a
request outside the occupied interval; between servers, the two neighbors
move at equal speed until one reaches the request. DC may use coincident
continuous positions, while DP and the prediction policies use discrete
$K$-element subsets. We retain this implementation difference explicitly
and use their measured costs, rather than importing a worst-case theorem
as an empirical guarantee. All reported dates have positive OPT. Mean
daily OPT for $K=3,4,5,6$ at 15 minutes is respectively
$10.1973,7.3589,4.4685,4.0712$; for $K=4$ at one minute it is $61.2164$.

\subsection{Predictors, landmark fitting, and prediction access}
\label{app:citibike-landmarks}

The mean predictor uses the same table on every evaluation date:
\[
    P_t^{\rm mean}(C)=\frac{1}{|\mathcal D_{\rm fit}|}
        \sum_{d\in\mathcal D_{\rm fit}}w_t^d(C),\qquad
    \widetilde P_t(C)=\min_{\ell\in L}\{P_t(\ell)+d(C,\ell)\}.
\]
The average of canonical values is $1$-Lipschitz. The extension interpolates
the retained predicted dual values, and $L=\mathcal C_K$ recovers Full. The terminal
prediction is explicitly zero for every method, including compressed ones.

Random uses nested prefixes of a seeded random permutation of all states.
Geometric starts at a distance medoid and adds the state farthest from its
nearest landmark. Distortion greedily minimizes
\begin{equation}
    \widehat J_{\rm recon}(L)=\frac1S\sum_{j=1}^S
        \bigl\|E_L(v^j|_L)-v^j\bigr\|_{\rm sp},\qquad S\le1024,
    \label{eq:citibike-selector}
\end{equation}
on sampled training day/slot exact-value vectors, reusing the same rows
throughout a greedy run. Each addition minimizes this empirical objective
among candidates; there is no global subset optimization or test-cost
minimization. The exact-value reconstruction error is nonnegative at every
state and zero at landmarks, so its span equals its maximum over states.
Distortion seeds 0--4 vary the sampled rows; Random seeds 0--9 vary the permutations.
Budgets use nested prefixes within each seed. These seeds do not resample
entire training datasets.

Equation~\eqref{eq:citibike-selector} selects landmarks using reconstruction
distortion, with predictors fitted separately; see \Cref{app:empirical-scope}
for its relation to the learning guarantees. Across the temporal predictors, the same original
landmark sets are reused within each split, with no predictor-specific
or evaluation-specific refitting of those sets.

Landmark identities and the known metric are outside the prediction budget.
The prediction budget is $m$ values per slot, reported as a fraction $m/n$.
The implementation computes full prediction vectors before restricting
them to landmarks; this budget measures values supplied to the policy,
not end-to-end computational, storage, or communication savings.

\subsection{Temporal predictors and structural Lipschitz outputs}
\label{app:temporal-predictors}

The temporal extension fixes $K=4$ and 15-minute requests. Static potential
averages the mean table over nonterminal times and holds that configuration
potential constant until the terminal zero. The calendar MLP receives six
features: sine/cosine pairs for time of day, weekday, and position in the
year, accounting for leap-year length. A 32-unit Tanh hidden layer is
concatenated with those calendar features before the $n$-value linear head.
The GRU \citep{cho2014learning} has one 32-unit recurrent layer, receiving
the same calendar features and an 11-class one-hot request vector (ten
bins and the empty label). Its hidden state and calendar features feed the
linear head. The MLP has 8,414 parameters and the GRU 13,086; this is not
a parameter-matched comparison.

The hidden state resets daily. In the one-based notation of
\eqref{eq:citibike-policy}, $P_t$ uses only the requests $r_1,\ldots,r_t$
already observed before that decision. Equivalently, the implementation's
zero-based $P_{j+1}$ follows the update on $r_j$. Predictions do not depend
on the policy's server state, so all compression variants receive the same
causally available predictor values. No feature scaler is fitted.

For arbitrary network scores $z_t$, the deterministic output map is
\begin{align}
 a_t(C)&=\min_{C'}\{z_t(C')+d(C,C')\}, &
 b_t(C)&=\max_{C'}\{z_t(C')-d(C,C')\},\nonumber\\
 q_t(C)&=\tfrac12[a_t(C)+b_t(C)], &
 P_t(C)&=q_t(C)-\frac1n\sum_{C'}q_t(C'),\qquad P_T=0.
 \label{eq:neural-envelope}
\end{align}
Every distance cone is $1$-Lipschitz; taking a finite minimum or maximum
preserves the common bound, as do averaging and subtracting a constant.
Consequently, every parameter setting and history yields spatially
$1$-Lipschitz values. If $z_t$ is already $1$-Lipschitz, both envelopes equal
$z_t$. This is a constraint map, not a claimed Euclidean projection or a
Lipschitz guarantee with respect to input histories.

The target is $Y_t^d(C)=w_t^d(C)-n^{-1}\sum_{C'}w_t^d(C')$. Training
minimizes mean squared error over dates, nonterminal times, and configurations.
AdamW uses learning rate $10^{-3}$, weight decay $10^{-4}$, 16 complete
days per batch, gradient clipping at 1.0, at most 100 epochs, and patience
12. Training is float32; at evaluation, raw scores are converted to
float64 before the output map, reconstruction, and rollout. Deterministic
algorithms are enabled and TF32 is disabled. Gradient clipping is unrelated
to the spatial guarantee. Each network uses training seeds 0, 1, and 2.
Validation MSE selects the epoch count, followed by a fresh refit on all
731 fitting days for that count. All selection runs reach the 100-epoch
cap; this fixed-budget protocol does not establish convergence to the best
possible model.

\subsection{Budgets and data sensitivity}
\label{app:citibike-protocol}

The mean-predictor suite uses nominal fractions
$\rho\in\{0.05,0.10,0.20,0.30,0.40,0.60,0.80,1.00\}$ and
$m=\max\{1,\operatorname{round}(\rho n)\}$. For $n=210$, the budgets are
$10,21,42,63,84,126,168,210$; the nominal 5\% budget is actually 4.76\%.
For $K=5$, the nominal 20\% budget is $50/252=19.84\%$.
The original five settings comprise four server counts at 15 minutes and
$K=4$ at one minute. The supplement adds $K=3,5,6$ at one minute, so all
eight settings have validation and annual evaluation runs.
The original predictor and landmark settings are frozen; validation
records diagnostics rather than supporting a broad hyperparameter search.
The temporal extension uses budgets $10,21,42,63,210$, crossing each of
three neural seeds with ten Random or five Distortion seeds.

Budgets count predicted values per slot, not per day. At a fixed $m$, the
one-minute horizon supplies 15 times as many values per episode as
the 15-minute horizon. Both traces come from the same raw trips, but
count-based modal aggregation changes the requests as well as the horizon
and OPT; it is not simple subsampling. Cross-resolution comparisons do
not hold daily communication volume or the request distribution fixed.

Seven further runs fix $K=4$, 15-minute requests, and $m=42$, using
$8,16,32,64,128,256,512$ fitting days. These are nested prefixes of one
seed-0 permutation of the 731-day fitting pool; the full-pool run supplies
the 731-day reference. The geographic bins and metric still use the initial
639-day calibration, which is not counted on the fitting-day axis. The
experiment measures sensitivity along one date-sampling path, not repeated
training-set sampling or empirical sample-complexity scaling.

Neural validation scores select the epoch count; final evaluation uses
separately refitted networks. Validation and annual evaluation have
different fitting pools and seasons.

The first ten evaluation dates (January 1--10, 2025) were used in preliminary
development, and additional analyses were developed after inspecting the
evaluation year. Excluding these dates is a sensitivity check rather than
a fresh holdout evaluation (\Cref{app:a100-sensitivity}).

\subsection{Metrics and statistical interpretation}
\label{app:citibike-statistics}

For each date, average all applicable training-seed and landmark-seed
combinations before averaging across dates:
\begin{equation}
 R_d=\frac{1}{|\mathcal S|}\sum_{s\in\mathcal S}
       \frac{\mathrm{ALG}_{d,s}}{\mathrm{OPT}_d},\qquad
 \bar R=\frac{1}{|\mathcal D_{\rm eval}|}\sum_d R_d,\qquad
 \operatorname{Ret}(m)=\frac{B-\bar R_m}{B-\bar R_{\rm Full}},
 \label{eq:citibike-metrics}
\end{equation}
where $B=\min\{\bar R_{\rm WFA},\bar R_{\rm DC}\}$ and performance retention
is defined only when Full improves on $B$. This is the mean of daily ratios, not a
ratio of annual total costs or an estimated worst-case competitive ratio.
Each date has equal weight regardless of OPT. Retention above 100\% means
lower observed cost than Full; negative retention means worse cost than
the classical reference. Its denominator depends on the predictor, so
cross-predictor comparisons must also use the actual costs.

With $B_t$ denoting the Bellman operator for the observed request, report
\begin{equation}
 \eta_d(P)=\sum_{t=1}^T\|B_tP_t-P_{t-1}\|_{\rm sp},\qquad
 E_{\rm recon}=\frac1T\sum_{t=0}^{T-1}
                   \|\widetilde P_t-P_t\|_{\rm sp}.
 \label{eq:citibike-errors}
\end{equation}
Full has zero reconstruction error relative to itself, but generally
nonzero residual. This reconstruction diagnostic is not the exact-value
distortion $\kappa_t(L;X)$ in the theory: it compares against predicted
Full values rather than $w_t$. Centered value RMSE and value-error span
are additional temporal diagnostics. For a fixed mean table, reconstruction
span is constant across evaluation dates, so a conditional date interval
can collapse to a point. These diagnostics need not rank predictors in the
same order as online cost.

Original-suite marginal intervals use 2,000 ordinary date-bootstrap
replicates. Its paired comparisons additionally use circular seven-day
blocks (2,000 replicates, seed 20260919). Temporal headline and paired
intervals use seven-day blocks (2,000 replicates, seed 20260921); ordinary
date intervals are also preserved in the experiment archive. Negative
paired differences favor the left-hand method. Supplementary mean,
cross-resolution and oracle intervals use 2,000 circular seven-day block
replicates with seed 20260922; ordinary date intervals are also saved.
Recomputed control intervals can differ slightly from the original
report because their bootstrap seed changes. Tables identify their source
protocol, and original observations are not pooled again as new evidence.
All intervals condition on the fixed training data and seed set;
they do not include uncertainty from
independently resampled training datasets. Blocks retain some short-term
dependence but do not eliminate seasonal or cross-year shifts. Comparisons
are unadjusted for multiplicity; crossing zero is not an equivalence test.

\subsection{Frozen-landmark oracle protocol and certificate indexing}
\label{app:oracle-protocol}

The exact-value diagnostic covers all eight server-count/resolution
settings, using the training-fitted landmark sets, budgets, numerical
precision and tie rules from \Cref{app:citibike-protocol,app:citibike-model}.
Random has ten selection seeds, Distortion five, and Geometric one.

\paragraph{Oracle inputs and interpretation.}
The oracle replaces predicted dual values at frozen landmarks by their
exact evaluation-day values. It uses the same McShane reconstruction,
terminal zero, empty-request convention, and value-greedy action rule.
Oracle Full must achieve OPT. At fixed $L$, predicted-compressed minus
oracle-compressed cost is the net effect of replacing predicted dual values;
it can have either sign. Oracle excess is the decision loss with exact
dual values and the chosen reconstruction. These comparisons do not
identify independent additive causal sources of prediction and compression
loss. In particular, fitted landmark sets, request distributions, horizons,
and OPT all differ between resolutions.

\paragraph{Distortion and certificate indexing.}
The report's zero-based $V_t^*$, defined before request $r_t$, is the paper's
$w_t$, defined after the first $t$ requests. For a fixed date and landmark
set, put $\widetilde w_t=E_L(w_t|_L)$ for $t<T$ and override
$\widetilde w_T=0$. Define
\begin{equation}
 e_t=\|\widetilde w_t-w_t\|_{\rm sp},\qquad
 e_T=0,\qquad S_0=\sum_{t=0}^{T-1}e_t,\qquad \bar e=S_0/T.
 \label{eq:oracle-distortion-indexing}
\end{equation}
For the decision times $1\le t<T$, $e_t$ is exactly the paper's
$\kappa_t(L;X)$ in \eqref{eq:distortion-definition}. The report's
$\kappa_\Sigma$ is $S_0$, including $e_0$; it is not the decision-time
sum $\sum_{t=1}^{T-1}\kappa_t=S_0-e_0$. Its reported ``mean distortion''
is $\bar e$, distinct from the predictor-relative reconstruction error
in \eqref{eq:citibike-errors}. Across resolutions, $S_0$ additionally
depends on the number of slots.

Let $\mathcal A=\{t\in\{1,\ldots,T\}:r_t\ne-1\}$ use the paper's
one-based request convention. The archived bounds become
\begin{align}
 0\le \mathrm{ALG}_{\rm or}-\mathrm{OPT}
    &\le \eta(\widetilde w)
     \le \underbrace{e_0+2\sum_{t=1}^{T-1}e_t}_{B_{\rm adj}=2S_0-e_0},
       \label{eq:oracle-residual-certificate}\\
 \mathrm{ALG}_{\rm or}-\mathrm{OPT}
    &\le \underbrace{\sum_{t\in\mathcal A}e_t}_{B_{\rm act}}
     \le S_0-e_0.
       \label{eq:oracle-active-certificate}
\end{align}
Here $\eta$ is defined in \eqref{eq:citibike-errors}.
The residual bound follows from \Cref{app:residual-comparison}:
$\|B_t\widetilde w_t-\widetilde w_{t-1}\|_{\rm sp}\le e_t+e_{t-1}$,
whose sum is $2S_0-e_0$ because $e_T=0$.
The active-request bound instead sums the exact Bellman gaps in
\eqref{eq:bellman_gap_span}. Empty requests leave the state unchanged
and contribute zero, while the terminal error is zero.
Thus \eqref{eq:oracle-active-certificate} is the value-distortion bound
for exact landmark values, with the known empty rounds removed.

\paragraph{Per-date checks and same-state action comparison.}
Bounds are checked for each date and selection seed before aggregation;
the reported entries average each date's bound divided by its OPT.
Oracle Full equality and zero residual and distortion at $m=n$ are
checked as well.

For $K=5$, an additional diagnostic compares predicted Full and Distortion
20\% actions at each state actually visited by predicted Full. If $a_t$
and $b_t$ are their two actions from that same state $s_{t-1}^{F}$, compute
\[
 \Delta Q_t=d(s_{t-1}^{F},b_t)+w_t(b_t)
             -d(s_{t-1}^{F},a_t)-w_t(a_t).
\]
Negative values favor the compressed action under exact continuation.
The reported statistic averages $\sum_t\Delta Q_t/\mathrm{OPT}$ across
dates and seeds. Because states are drawn from Full's trajectory, this
sum is not the difference in the two policies' realized trajectory costs.

\subsection{Synthetic pilot data and split independence}
\label{app:data-generation}

The metric is $M=\{x_j=j/8:j=0,\ldots,8\}$ with
$d(x_i,x_j)=|x_i-x_j|$, diameter $D=1$, and $s_0=x_4=0.5$.
Every episode has $T=12$ tasks. For each of the seeds $7,19,41$, a
fresh NumPy \texttt{default\_rng(seed)} generator creates the training,
validation, and test splits sequentially, with sizes $32$, $64$, and
$256$. Each scenario restarts the generator with the same seed values.
No external dataset or preprocessing is involved.

The service costs are
\begin{equation}
 c_{it}(x_j)=A_{it}(x_j-\mu_{it})^2+U_{itj},\qquad
 \mu_{it}=\operatorname{clip}_{[0,1]}(\bar\mu_t+\xi_i+\zeta_{it}),
 \label{eq:experiment_costs}
\end{equation}
where $\xi_i\sim\mathcal N(0,0.045^2)$,
$\zeta_{it}\sim\mathcal N(0,0.035^2)$, and
$U_{itj}\sim\mathrm{Unif}[0,0.025]$.
All underlying draws
are independent, except that the same episode shift is deliberately
shared across its tasks. The preferred center is clipped to $[0,1]$;
the resulting costs are already nonnegative and receive no further
normalization. Generator phase $t=0,\ldots,11$ corresponds to paper
task $t+1$.

\begin{table}[htbp]
\centering
\small
\caption{Complete pilot configuration. Gaussian scales are standard
deviations. Both scenarios use the same split sizes and landmark budget.}
\label{tab:appendix-configuration}
\begin{tabular}{@{}lll@{}}
\toprule
Quantity & Localized & Switching \\
\midrule
States / horizon / pair budget & $9\,/\,12\,/\,2$ & $9\,/\,12\,/\,2$ \\
Train / validation / test per seed & $32\,/\,64\,/\,256$ & $32\,/\,64\,/\,256$ \\
Preferred-center template $\bar\mu_t$
    & $0.68+0.11\sin(2\pi t/6)$ & $0.50+0.30\sin(2\pi t/6)$ \\
Episode / task Gaussian scales & $0.045\,/\,0.035$ & $0.045\,/\,0.035$ \\
Amplitude distribution & $\mathrm{Unif}[1.5,2.5]$ & $\mathrm{Unif}[4,7]$ \\
Per-state additive noise & $\mathrm{Unif}[0,0.025]$ & $\mathrm{Unif}[0,0.025]$ \\
Enumerated pair candidates & $\binom92=36$ & $\binom92=36$ \\
\bottomrule
\end{tabular}
\end{table}

\subsection{Label computation, fitting, and tie rules}
\label{app:implementation-details}

The label routine \texttt{bellman} evaluates
\eqref{eq:bellman_recursion} in double precision. It stores a
$(T+1)\times n$ array with a zero terminal row. The fitting routine
anchors the first landmark in each sorted set and solves
\eqref{eq:span-erm-lp} using SciPy's \texttt{linprog} with
\texttt{method='highs'}. The implementation
uses sparse constraint matrices, the solver's default options, and checks solver success;
there is no tuning on validation or test data. The returned table has
shape $(T+1)\times m$, with only rows $1,\ldots,T-1$ optimized and
unused time-zero and terminal rows set to zero.
All-zero tables are returned directly for a singleton.

The pilot constructs pairs in lexicographic order. NumPy's first-index
minimum rule resolves exact ties in pair selection and online actions;
states are ordered by index. If the LP has multiple minimizers, the
solver-returned optimum is used without a secondary objective. The
locked environment records the implementation that produced the
archived outputs; different solver versions can choose different
optimal tables.

\paragraph{Selection methods.}
Each of the 36 pairs is fitted using only the training split.
The distortion method adds the empirical $\kappa$ average to the
LP's empirical span loss and chooses a minimizing pair, implementing
\Cref{thm:joint-learning} up to numerical optimization error.
The validation method freezes those fitted tables and chooses the
pair with the lowest mean validation excess, as in
\Cref{thm:validation-selection}. The geometric baseline minimizes,
in order, covering radius, total distance of all states to the pair,
and the pair's lexicographic index; it always selects
$\{x_1,x_6\}=\{1/8,3/4\}$.
The random baseline is the arithmetic average of the test means of
all 36 fitted pairs. It represents the exact expected performance
of a uniformly drawn pair and does not select using test outcomes.

The singleton baseline compares all nine anchored zero-table policies
on validation episodes. The full-state baseline fits $L=M$ using the
same training split and LP, with no additional input information.
All selection is completed before test performance is inspected.

\paragraph{Information available to each policy.}
A fitted rollout uses only its fixed table, the metric, its current
state, and the current service-cost vector. Exact test values are used
offline for $\mathrm{OPT}$ and the certificates below. The code's
explicit \texttt{v=None} mode instead supplies exact landmark values
at each time; only the rows labeled oracle use this mode.
Every rollout uses $\widehat w_T=0$ on all states.

\subsection{Metrics, certificates, and aggregation}
\label{app:evaluation-metrics}

For every test episode, the primary metric is total episode excess
$e(X)=\mathrm{ALG}(X)-w_0^X(s_0)$, not excess divided by $T$ or
by $\mathrm{OPT}$. For seed $s$, let
$\bar e_s=256^{-1}\sum_{i=1}^{256}e(X_{si})$. The entries in
\Cref{tab:pilot_results} are
\begin{equation}
    \bar e=\frac13\sum_s\bar e_s,
    \qquad
    \operatorname{sd}_{\rm seed}
       =\sqrt{\frac1{3-1}\sum_s(\bar e_s-\bar e)^2}.
    \label{eq:seed-aggregation}
\end{equation}
Thus the uncertainty shown is the sample standard deviation of three
seed means, not a confidence interval or the standard deviation of
the 768 pooled episode costs. No hypothesis test is reported.

The recorded offline certificates are
\begin{equation}
\begin{aligned}
    C_\kappa(X)&=\sum_{t=1}^{\tau}
           [\kappa_t(L;X)+2\delta_t],\\
    C_{\rm loc}(X)&=\sum_{t=1}^{\tau}
           \min\{2D,\,2\rho_t^\star(s_{t-1};L)+2\delta_t\}.
\end{aligned}
\label{eq:recorded-certificates}
\end{equation}
Their validity follows from
\Cref{thm:distortion-upper-bound,thm:local-upper-bound} and the
diameter bound. The implementation identifies numerical Bellman
minimizers with \texttt{isclose} using absolute and relative tolerances
$10^{-10}$. Online actions themselves use the unmodified first-index
argmin. Pilot certificate inequalities are checked with additive
tolerance $10^{-8}$; these checks accommodate floating-point arithmetic
rather than changing the mathematical definition of a successor.

\subsection{Software and reproduction}
\label{app:reproduction}
\label{app:citibike-reproduction}

The Citi Bike runs used Python 3.11.6, NumPy 1.26.4, pandas 2.3.3,
PyTorch 2.5.1 and CUDA 11.8. Mean-predictor runs used two A100 80GB GPUs,
neural training used one, and additional resolution/oracle runs used CPUs
with four threads per process. DP and evaluation use float64; neural
training and greedy landmark scoring use float32. Synthetic DP and LP runs used
Linux x86\_64 CPUs with Python 3.13.5, NumPy 2.3.5, SciPy 1.17.0 and
Matplotlib 3.10.8.

The accompanying code package's \path{README.md} and
\path{docs/REPRODUCING.md} describe setup, data preparation and execution.
Run \texttt{bash scripts/reproduce\_synthetic.sh} to rerun both synthetic
pilots and implementation checks, and compare with
\path{reference_results/synthetic/}. Raw Citi Bike data require a separate
download; trained weights are not included. In the paper repository,
\path{scripts/build_a100_assets.py} regenerates tables, plots and supplementary
CSV files from the frozen report tables in \path{data/}, without retraining
or recomputing date-level statistics. \path{supplementary/README.md} indexes
the detailed numerical results and implementation records.

\section{Additional Experimental Results}
\suppressfloats[t]
\label{app:additional-results}

We report budget and resolution effects, compression across predictors,
exact-value diagnostics, sensitivity checks, and synthetic fixed-table
learning under the protocols in \Cref{app:experiment-details}.
Complete numerical sweeps and implementation records accompany the paper
in the supplementary files described in \Cref{app:citibike-reproduction}.

\subsection{Prediction budgets and resolution}
\label{app:increment-resolution}

\Cref{fig:a100-complete-budgets} combines the predicted and exact-value
budget sweeps across all eight settings. Small budgets can fail: for $K=4$ at
15 minutes, retaining $10/210$ predicted values gives Distortion cost
ratio $1.8126$, above WFA's $1.7067$. At $63/210$, Distortion and
Geometric reach $1.1527$ and $1.1528$, compared with Full's $1.1524$.
The curves also retain nonmonotone behavior across budgets.

\begin{figure}[htbp]
    \centering
    \input{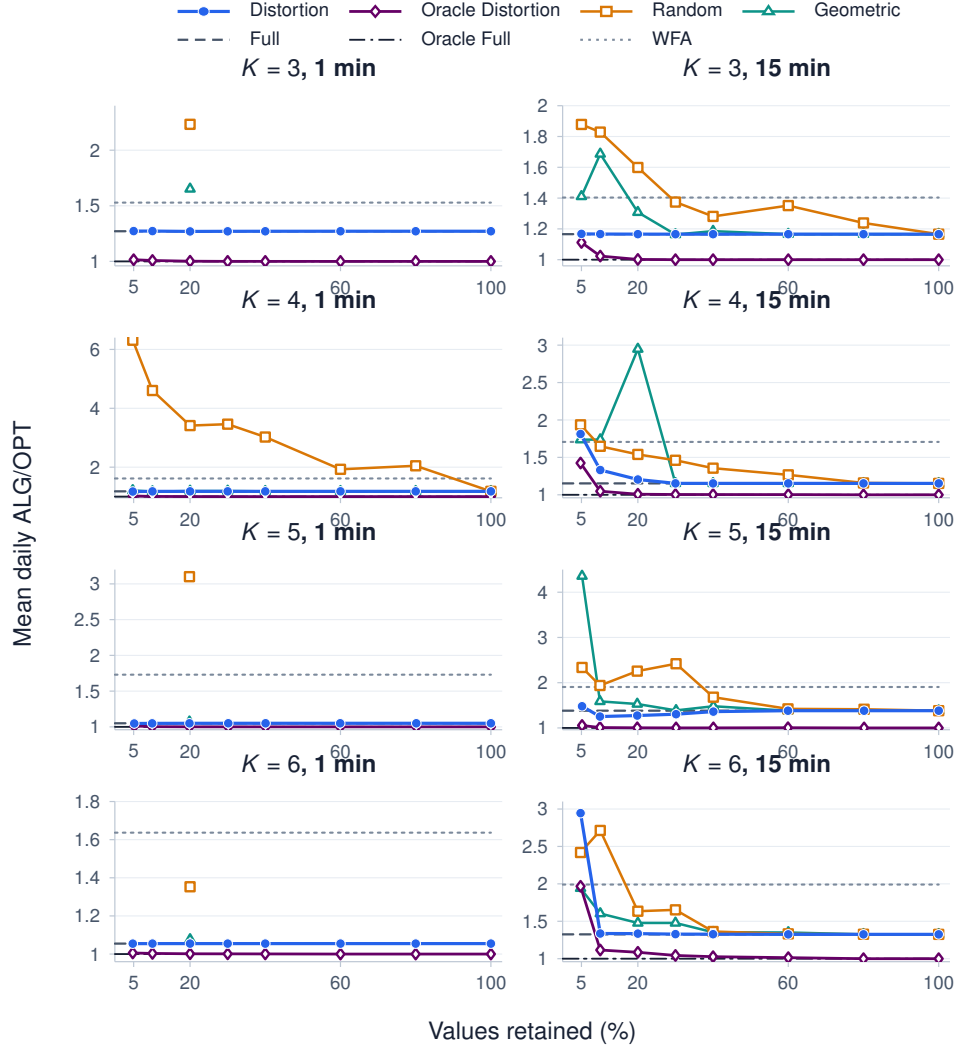}
    \caption{Predicted and exact-value budget--cost curves across all eight
    settings. Distortion uses the same frozen landmarks for predicted and
    oracle inputs; Oracle Full equals OPT. Random and Geometric have complete
    sweeps in five settings and only the reported 20\% points in the other
    three one-minute settings. Lines connect supplied observations; axis
    ranges differ. Exact-value results are discussed in \Cref{app:increment-oracle-results}.}
    \label{fig:a100-complete-budgets}
    \label{fig:increment-oracle}
\end{figure}

At approximately 20\%, the one-minute Distortion--Full mean differences
for $K=3,4,5,6$ are $-0.001747,-0.000142,-0.000293,0$.
\Cref{tab:increment-resolution} gives their paired intervals and
cross-resolution contrasts; each resolution is normalized by its own OPT.
For $K=5$ at one minute, the block interval
$[-0.000662,-0.000011]$ excludes zero, whereas ordinary date resampling
gives $[-0.000680,+0.000009]$. The apparent significance of this tiny
difference is sensitive to the resampling scheme. For $K=6$, the recorded
zero gap and $[0,0]$ interval apply to the evaluated dates and seeds.
\Cref{tab:a100-paired} reports comparisons against Random and Geometric
for the five settings with archived paired selector comparisons.

\begin{table}[htbp]
\centering\small
\setlength{\tabcolsep}{4pt}
\caption{Distortion minus Full with approximately 20\% of predicted values retained, and its paired cross-resolution contrast. Each resolution uses its own daily OPT. All intervals here use circular seven-day blocks, 2,000 replicates, seed 20260922.}\label{tab:increment-resolution}
\begin{tabular}{@{}llrr@{}}
\toprule
$K$ & Resolution / contrast & $\Delta$ or contrast & 95\% block CI \\
\midrule
3 & 1 min & -0.001747 & [-0.003456, -0.000187] \\
3 & 15 min & +0.000169 & [+0.000000, +0.000506] \\
3 & 1 min minus 15 min & -0.001916 & [-0.003676, -0.000297] \\
4 & 1 min & -0.000142 & [-0.000862, +0.000551] \\
4 & 15 min & +0.0518 & [+0.0306, +0.0728] \\
4 & 1 min minus 15 min & -0.0520 & [-0.0729, -0.0306] \\
5 & 1 min & -0.000293 & [-0.000662, -0.000011] \\
5 & 15 min & -0.1074 & [-0.1733, -0.0401] \\
5 & 1 min minus 15 min & +0.1071 & [+0.0401, +0.1730] \\
6 & 1 min & +0.000000 & [+0.000000, +0.000000] \\
6 & 15 min & +0.0097 & [-0.0075, +0.0265] \\
6 & 1 min minus 15 min & -0.0097 & [-0.0265, +0.0075] \\
\bottomrule
\end{tabular}
\end{table}

\begin{table}[htbp]
\centering\small
\setlength{\tabcolsep}{4pt}
\caption{Paired selector differences at approximately 20\%. Negative differences favor Distortion. Intervals use 2,000 seven-day block replicates (seed 20260919) and are unadjusted for multiple comparisons. Distortion--Full intervals are in \Cref{tab:increment-resolution}.}\label{tab:a100-paired}
\begin{tabular}{@{}llrr@{}}
\toprule
Setting & Comparison & Difference & 95\% block CI \\
\midrule
$K=3$, 15 min & Distortion - Random & -0.43294 & [-0.4845, -0.3891] \\
$K=3$, 15 min & Distortion - Geom. & -0.14221 & [-0.2242, -0.0664] \\
$K=4$, 15 min & Distortion - Random & -0.33533 & [-0.3843, -0.2904] \\
$K=4$, 15 min & Distortion - Geom. & -1.74355 & [-1.8554, -1.6376] \\
$K=5$, 15 min & Distortion - Random & -0.98201 & [-1.0892, -0.8776] \\
$K=5$, 15 min & Distortion - Geom. & -0.25517 & [-0.3431, -0.1614] \\
$K=6$, 15 min & Distortion - Random & -0.29947 & [-0.3421, -0.2657] \\
$K=6$, 15 min & Distortion - Geom. & -0.14265 & [-0.1835, -0.1017] \\
$K=4$, 1 min & Distortion - Random & -2.23429 & [-2.3398, -2.1324] \\
$K=4$, 1 min & Distortion - Geom. & -0.01089 & [-0.0201, -0.0015] \\
\bottomrule
\end{tabular}
\end{table}

\paragraph{Reconstruction and Bellman-residual diagnostics.}
\label{app:a100-paired-results}
\Cref{tab:a100-errors} compares Full and Distortion reconstruction span
and Bellman residual. Complete four-method diagnostics and marginal
intervals are retained in \path{supplementary/results/a100_errors.csv}.
For $K=5$, Distortion reduces ALG/OPT from $1.3810$ to $1.2737$ even
though normalized residual rises from $15.5573$ to $39.5873$.
These diagnostics need not rank policies by online cost. Reconstruction
span is measured against predicted Full values; exact-value distortion
and its cost certificates are evaluated in
\Cref{app:increment-oracle-results}.

\begin{table}[htbp]
\centering\small
\setlength{\tabcolsep}{4pt}
\caption{Mean-predictor diagnostics for Full and approximately 20\% Distortion. Recon. is the per-slot reconstruction span relative to predicted Full values. The supplementary CSV retains all four methods and their ordinary date-bootstrap marginal intervals.}\label{tab:a100-errors}
\begin{tabular}{@{}llrrr@{}}
\toprule
Setting & Method & ALG/OPT & $\eta$/OPT & Recon. \\
\midrule
$K=3$, 15 min & Full & 1.1657 & 7.6815 & 0.0000 \\
$K=3$, 15 min & Distortion & 1.1659 & 9.4318 & 0.8925 \\
$K=4$, 15 min & Full & 1.1524 & 9.7615 & 0.0000 \\
$K=4$, 15 min & Distortion & 1.2043 & 18.2262 & 1.4162 \\
$K=5$, 15 min & Full & 1.3810 & 15.5573 & 0.0000 \\
$K=5$, 15 min & Distortion & 1.2737 & 39.5873 & 1.6323 \\
$K=6$, 15 min & Full & 1.3252 & 14.2701 & 0.0000 \\
$K=6$, 15 min & Distortion & 1.3349 & 44.0859 & 1.9097 \\
$K=4$, 1 min & Full & 1.1788 & 8.8766 & 0.0000 \\
$K=4$, 1 min & Distortion & 1.1786 & 9.9711 & 0.1928 \\
\bottomrule
\end{tabular}
\end{table}

\FloatBarrier
\subsection{Compression across predictors}
\label{app:a100-temporal-results}

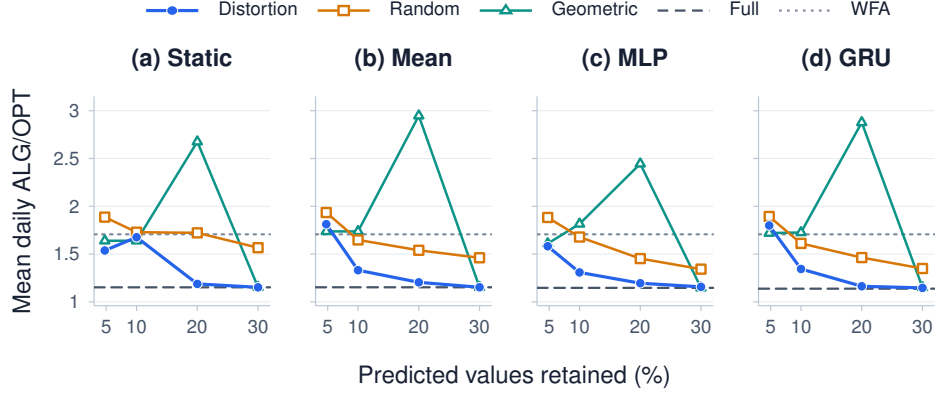
\begin{figure}[htbp]
    \centering
    \begin{tikzpicture}
\begin{groupplot}[group style={group size=4 by 1,horizontal sep=0.60cm},paper plot,ylabel={Mean daily ALG/OPT},group/ylabels at=edge left,width=0.28\linewidth,height=4.35cm,xmin=3,xmax=32,xtick={5,10,20,30},ymin=0.96,ymax=3.15,ytick={1,1.5,2,2.5,3},group/yticklabels at=edge left,legend columns=5,legend style={at={(2.42,1.34)},anchor=south}]
\nextgroupplot[title={(a) Static},]
\addplot[color=plotSlate,line width=0.85pt,dash pattern=on 4pt off 2.5pt,no marks,line cap=round,line join=round,forget plot] coordinates {(3.000000,1.1520) (32.000000,1.1520)};

\addplot[color=plotMuted,line width=0.8pt,dash pattern=on 1pt off 2pt,no marks,line cap=round,line join=round,forget plot] coordinates {(3.000000,1.7067) (32.000000,1.7067)};

\addplot[color=plotTeal,line width=0.95pt,mark=triangle*,mark size=2.0pt,mark options={solid,fill=white,line width=0.9pt},line cap=round,line join=round,forget plot] coordinates {(4.761905,1.6394) (10.000000,1.6394) (20.000000,2.6758) (30.000000,1.1533)};

\addplot[color=plotAmber,line width=0.95pt,mark=square*,mark size=1.7pt,mark options={solid,fill=white,line width=0.9pt},line cap=round,line join=round,forget plot] coordinates {(4.761905,1.8866) (10.000000,1.7290) (20.000000,1.7223) (30.000000,1.5667)};

\addplot[color=plotBlue,line width=1.2pt,mark=*,mark size=1.8pt,mark options={solid,fill=plotBlue,draw=white,line width=0.35pt},line cap=round,line join=round,forget plot] coordinates {(4.761905,1.5371) (10.000000,1.6760) (20.000000,1.1877) (30.000000,1.1520)};

\addlegendimage{color=plotBlue,line width=1.2pt,mark=*,mark size=1.8pt,mark options={solid,fill=plotBlue,draw=white,line width=0.35pt}}
\addlegendentry{Distortion}
\addlegendimage{color=plotAmber,line width=0.95pt,mark=square*,mark size=1.7pt,mark options={solid,fill=white,line width=0.9pt}}
\addlegendentry{Random}
\addlegendimage{color=plotTeal,line width=0.95pt,mark=triangle*,mark size=2.0pt,mark options={solid,fill=white,line width=0.9pt}}
\addlegendentry{Geometric}
\addlegendimage{color=plotSlate,line width=0.85pt,dash pattern=on 4pt off 2.5pt,no marks}
\addlegendentry{Full}
\addlegendimage{color=plotMuted,line width=0.8pt,dash pattern=on 1pt off 2pt,no marks}
\addlegendentry{WFA}

\nextgroupplot[title={(b) Mean},]
\addplot[color=plotSlate,line width=0.85pt,dash pattern=on 4pt off 2.5pt,no marks,line cap=round,line join=round,forget plot] coordinates {(3.000000,1.1524) (32.000000,1.1524)};

\addplot[color=plotMuted,line width=0.8pt,dash pattern=on 1pt off 2pt,no marks,line cap=round,line join=round,forget plot] coordinates {(3.000000,1.7067) (32.000000,1.7067)};

\addplot[color=plotTeal,line width=0.95pt,mark=triangle*,mark size=2.0pt,mark options={solid,fill=white,line width=0.9pt},line cap=round,line join=round,forget plot] coordinates {(4.761905,1.7373) (10.000000,1.7373) (20.000000,2.9478) (30.000000,1.1528)};

\addplot[color=plotAmber,line width=0.95pt,mark=square*,mark size=1.7pt,mark options={solid,fill=white,line width=0.9pt},line cap=round,line join=round,forget plot] coordinates {(4.761905,1.9350) (10.000000,1.6483) (20.000000,1.5396) (30.000000,1.4614)};

\addplot[color=plotBlue,line width=1.2pt,mark=*,mark size=1.8pt,mark options={solid,fill=plotBlue,draw=white,line width=0.35pt},line cap=round,line join=round,forget plot] coordinates {(4.761905,1.8126) (10.000000,1.3311) (20.000000,1.2043) (30.000000,1.1527)};

\nextgroupplot[title={(c) MLP},]
\addplot[color=plotSlate,line width=0.85pt,dash pattern=on 4pt off 2.5pt,no marks,line cap=round,line join=round,forget plot] coordinates {(3.000000,1.1464) (32.000000,1.1464)};

\addplot[color=plotMuted,line width=0.8pt,dash pattern=on 1pt off 2pt,no marks,line cap=round,line join=round,forget plot] coordinates {(3.000000,1.7067) (32.000000,1.7067)};

\addplot[color=plotTeal,line width=0.95pt,mark=triangle*,mark size=2.0pt,mark options={solid,fill=white,line width=0.9pt},line cap=round,line join=round,forget plot] coordinates {(4.761905,1.6099) (10.000000,1.8156) (20.000000,2.4448) (30.000000,1.1472)};

\addplot[color=plotAmber,line width=0.95pt,mark=square*,mark size=1.7pt,mark options={solid,fill=white,line width=0.9pt},line cap=round,line join=round,forget plot] coordinates {(4.761905,1.8831) (10.000000,1.6784) (20.000000,1.4528) (30.000000,1.3423)};

\addplot[color=plotBlue,line width=1.2pt,mark=*,mark size=1.8pt,mark options={solid,fill=plotBlue,draw=white,line width=0.35pt},line cap=round,line join=round,forget plot] coordinates {(4.761905,1.5819) (10.000000,1.3078) (20.000000,1.1950) (30.000000,1.1569)};

\nextgroupplot[title={(d) GRU},]
\addplot[color=plotSlate,line width=0.85pt,dash pattern=on 4pt off 2.5pt,no marks,line cap=round,line join=round,forget plot] coordinates {(3.000000,1.1378) (32.000000,1.1378)};

\addplot[color=plotMuted,line width=0.8pt,dash pattern=on 1pt off 2pt,no marks,line cap=round,line join=round,forget plot] coordinates {(3.000000,1.7067) (32.000000,1.7067)};

\addplot[color=plotTeal,line width=0.95pt,mark=triangle*,mark size=2.0pt,mark options={solid,fill=white,line width=0.9pt},line cap=round,line join=round,forget plot] coordinates {(4.761905,1.7197) (10.000000,1.7260) (20.000000,2.8774) (30.000000,1.1510)};

\addplot[color=plotAmber,line width=0.95pt,mark=square*,mark size=1.7pt,mark options={solid,fill=white,line width=0.9pt},line cap=round,line join=round,forget plot] coordinates {(4.761905,1.8932) (10.000000,1.6117) (20.000000,1.4633) (30.000000,1.3489)};

\addplot[color=plotBlue,line width=1.2pt,mark=*,mark size=1.8pt,mark options={solid,fill=plotBlue,draw=white,line width=0.35pt},line cap=round,line join=round,forget plot] coordinates {(4.761905,1.7995) (10.000000,1.3441) (20.000000,1.1630) (30.000000,1.1460)};

\end{groupplot}
\node[anchor=north,text=plotInk,font=\fontfamily{phv}\selectfont\footnotesize] at ($(group c1r1.south west)!0.5!(group c4r1.south east)+(0,-19pt)$) {Predicted values retained (\%)};
\end{tikzpicture}
    \caption{Compression across predictors for $K=4$ at 15 minutes.
    Panels share a vertical scale and the styles in \Cref{fig:a100-budget}.
    Points show all four tested compressed budgets; every method recovers
    Full at $m=n$, omitted here.}
    \label{fig:a100-temporal}
\end{figure}

\Cref{tab:a100-temporal-full,tab:a100-temporal-20,tab:a100-temporal-paired}
separate Full-predictor quality from landmark selection.
The GRU has the lowest Full cost point estimate, but its paired intervals
against the mean table, MLP and static potential all cross zero.
At the 20\% budget, its Distortion--Random and Distortion--Geometric
intervals are wholly negative. Thus the evidence for landmark selection
is clearer than the evidence for improving the predictor itself.
All training and landmark seeds are averaged; none is selected on test cost.

\begin{table}[htbp]
\centering\small
\setlength{\tabcolsep}{4pt}
\caption{Full-predictor costs and value errors on the 365 evaluation dates. RMSE uses centered exact values; value-error span and Bellman residual are distinct diagnostics.}\label{tab:a100-temporal-full}
\begin{tabular}{@{}lrrrrr@{}}
\toprule
Predictor & Full & 95\% block CI & RMSE & Error span & $\eta$/OPT \\
\midrule
Static & 1.1520 & [1.1202, 1.1863] & 1.0403 & 4.3440 & 12.4987 \\
Mean & 1.1524 & [1.1214, 1.1865] & 0.7726 & 2.9219 & 9.7615 \\
MLP & 1.1464 & [1.1207, 1.1736] & 0.7856 & 3.2413 & 15.5222 \\
GRU & 1.1378 & [1.1139, 1.1632] & 0.7649 & 3.0593 & 14.5484 \\
\bottomrule
\end{tabular}
\end{table}

\begin{table}[htbp]
\centering\small
\setlength{\tabcolsep}{4pt}
\caption{Temporal-family compression retaining $m=42$ of 210 predicted values. Neural results average the crossed training and landmark seeds within each date. Bold marks the lowest ALG/OPT within each predictor, without implying statistical significance.}\label{tab:a100-temporal-20}
\begin{tabular}{@{}llrrrr@{}}
\toprule
Predictor & Method & ALG/OPT & 95\% block CI & Ret. & Recon. \\
\midrule
Static & Distortion & \textbf{1.1877} & [1.1576, 1.2201] & 93.57\% & 1.5840 \\
Static & Random & 1.7223 & [1.6565, 1.7925] & -2.81\% & 3.9430 \\
Static & Geom. & 2.6758 & [2.5404, 2.8136] & -174.69\% & 3.2623 \\
Mean & Distortion & \textbf{1.2043} & [1.1826, 1.2259] & 90.65\% & 1.4162 \\
Mean & Random & 1.5396 & [1.4922, 1.5905] & 30.15\% & 3.9819 \\
Mean & Geom. & 2.9478 & [2.8428, 3.0563] & -223.90\% & 3.0379 \\
MLP & Distortion & \textbf{1.1950} & [1.1752, 1.2156] & 91.33\% & 1.5246 \\
MLP & Random & 1.4528 & [1.4087, 1.5016] & 45.31\% & 3.9612 \\
MLP & Geom. & 2.4448 & [2.3552, 2.5323] & -131.72\% & 2.9447 \\
GRU & Distortion & \textbf{1.1630} & [1.1386, 1.1884] & 95.57\% & 1.5136 \\
GRU & Random & 1.4633 & [1.4208, 1.5119] & 42.78\% & 3.9983 \\
GRU & Geom. & 2.8774 & [2.7675, 2.9905] & -205.76\% & 2.9441 \\
\bottomrule
\end{tabular}
\end{table}

\begin{table}[htbp]
\centering\small
\setlength{\tabcolsep}{4pt}
\caption{Paired temporal-family comparisons. Intervals use seven-day block resampling and condition on the fitted data and seeds.}\label{tab:a100-temporal-paired}
\begin{tabular}{@{}lrr@{}}
\toprule
Comparison & Difference & 95\% block CI \\
\midrule
GRU Full - Mean Full & -0.01463 & [-0.0446, +0.0127] \\
GRU Full - MLP Full & -0.00860 & [-0.0210, +0.0034] \\
GRU Full - Static Full & -0.01417 & [-0.0442, +0.0134] \\
MLP Full - Mean Full & -0.00604 & [-0.0328, +0.0182] \\
Static Full - Mean Full & -0.00046 & [-0.0019, +0.0009] \\
Static: Distortion 20\% - Full & +0.03569 & [+0.0281, +0.0436] \\
Mean: Distortion 20\% - Full & +0.05182 & [+0.0302, +0.0730] \\
MLP: Distortion 20\% - Full & +0.04859 & [+0.0273, +0.0695] \\
GRU: Distortion 20\% - Full & +0.02522 & [+0.0100, +0.0420] \\
GRU: Distortion 20\% - Random 20\% & -0.30032 & [-0.3495, -0.2571] \\
GRU: Distortion 20\% - Geom. 20\% & -1.71436 & [-1.8271, -1.6035] \\
\bottomrule
\end{tabular}
\end{table}

Static potential achieves a mean daily ALG/OPT close to the mean table's
in this setting, despite removing its time variation and having larger
centered RMSE and value-error span.
The GRU reduces centered MSE by 1.97\% relative to the mean table,
but its value-error span and Bellman residual increase. Training minimizes
MSE, so improvement in all error measures is neither enforced nor observed.
The mean-predictor rows are shared with the preceding comparisons.

\subsection{Exact future-value compression and theoretical certificates}
\label{app:increment-oracle-results}

\Cref{tab:increment-oracle} holds the landmark sets fixed while replacing
their predicted dual values by exact evaluation-day values.
Oracle Full equals OPT; using exact values only at landmarks can still incur excess.
For $K=6$ at 15 minutes, Oracle Distortion at approximately 20\% has cost
ratio $1.0854$, with 95\% block interval $[1.0738,1.0970]$.
The corresponding one-minute ratio is $1.0014$. For $K=4$, the per-slot
exact distortion is $0.1423$ at one minute versus $1.3849$ at 15 minutes,
and oracle ratios are $1.0015$ and $1.0077$. The resolution contrast thus
persists under exact dual-value inputs, although this comparison does not
separate the effects of demand, horizon, or the fitted landmark sets.

\begin{table}[htbp]
\centering\small
\setlength{\tabcolsep}{4pt}
\caption{Frozen-landmark exact-value diagnostic at approximately 20\%. Oracle Full equals OPT in every setting. Intervals use seven-day blocks. The last column replaces inputs at fixed landmarks; it is not an independent causal component of total loss.}\label{tab:increment-oracle}
\begin{tabular}{@{}lrr@{}}
\toprule
Setting & Oracle Distortion [95\% CI] & Predicted minus oracle [95\% CI] \\
\midrule
$K=3$, 1 min & 1.0017 [1.0013, 1.0022] & +0.2678 [+0.2524, +0.2844] \\
$K=3$, 15 min & 1.0016 [1.0007, 1.0026] & +0.1643 [+0.1404, +0.1889] \\
$K=4$, 1 min & 1.0015 [1.0009, 1.0022] & +0.1771 [+0.1627, +0.1939] \\
$K=4$, 15 min & 1.0077 [1.0050, 1.0108] & +0.1966 [+0.1760, +0.2191] \\
$K=5$, 1 min & 1.0005 [1.0001, 1.0011] & +0.0494 [+0.0420, +0.0572] \\
$K=5$, 15 min & 1.0030 [1.0014, 1.0050] & +0.2707 [+0.2307, +0.3187] \\
$K=6$, 1 min & 1.0014 [1.0001, 1.0033] & +0.0534 [+0.0441, +0.0634] \\
$K=6$, 15 min & 1.0854 [1.0738, 1.0970] & +0.2495 [+0.2077, +0.2897] \\
\bottomrule
\end{tabular}
\end{table}

The $K=5$ same-state diagnostic in \Cref{tab:increment-actions} evaluates
action changes on the states visited by predicted Full. At 15 minutes,
the net exact-continuation difference is $-0.0832$, with interval
$[-0.1311,-0.0360]$, supporting the explanation that compression can alter
decisions induced by an imperfect Full predictor beneficially. Improved
actions are not more frequent than worsened actions; their magnitudes
also matter. At one minute, actions differ only about 0.01 times per day
and the net difference interval crosses zero. This diagnostic is
trajectory-conditioned, not a causal or additive decomposition of total cost.

\begin{table}[htbp]
\centering\small
\setlength{\tabcolsep}{4pt}
\caption{$K=5$ actions compared at the same states visited by predicted Full. Counts are daily averages; the final column averages the daily normalized sum of exact-continuation action differences. This is not the cost difference between the two policies\textquotesingle\ actual trajectories.}\label{tab:increment-actions}
\begin{tabular}{@{}lrrrr@{}}
\toprule
Minutes & Changed & Improved & Worsened & $\sum\Delta Q/\mathrm{OPT}$ [95\% CI] \\
\midrule
1 min & 0.010 & 0.004 & 0.002 & -0.000147 [-0.000493, +0.000110] \\
15 min & 0.992 & 0.336 & 0.348 & -0.0832 [-0.1311, -0.0360] \\
\bottomrule
\end{tabular}
\end{table}

\begin{table}[htbp]
\centering\small
\setlength{\tabcolsep}{4pt}
\caption{Exact-value distortion and certificates for Distortion at approximately 20\%. Normalize within each date before averaging. The final three columns are cost-ratio upper bounds, not attained costs; \Cref{app:oracle-protocol} defines $S_0,B_{\rm adj},B_{\rm act}$.}\label{tab:increment-bounds}
\begin{tabular}{@{}lrrrrrr@{}}
\toprule
Setting & $\bar e$ & $S_0/\mathrm{OPT}$ & $R_{\rm or}$ & $1+\eta/\mathrm{OPT}$ & $1+B_{\rm adj}/\mathrm{OPT}$ & $1+B_{\rm act}/\mathrm{OPT}$ \\
\midrule
$K=3$, 1 min & 0.1848 & 1.450 & 1.0017 & 2.043 & 3.900 & 2.403 \\
$K=3$, 15 min & 0.7715 & 8.138 & 1.0016 & 5.207 & 17.238 & 9.099 \\
$K=4$, 1 min & 0.1423 & 3.506 & 1.0015 & 3.230 & 8.010 & 4.408 \\
$K=4$, 15 min & 1.3849 & 19.542 & 1.0077 & 16.962 & 40.004 & 20.458 \\
$K=5$, 1 min & 0.1725 & 9.965 & 1.0005 & 6.581 & 20.929 & 10.815 \\
$K=5$, 15 min & 1.6489 & 41.135 & 1.0030 & 35.985 & 82.980 & 41.838 \\
$K=6$, 1 min & 0.2214 & 20.694 & 1.0014 & 10.512 & 42.384 & 21.406 \\
$K=6$, 15 min & 1.9219 & 50.474 & 1.0854 & 40.255 & 101.540 & 51.058 \\
\bottomrule
\end{tabular}
\end{table}

\Cref{app:oracle-protocol} defines the distortion indexing and recorded
certificates. The bounds in
\Cref{tab:increment-bounds} hold separately for every recorded seed/date,
but are generally much larger than realized excess. For example,
$K=4$ at one minute has actual oracle ratio $1.0015$, versus residual-based
bound $3.230$ and active-round distortion bound $4.408$.
These checks connect the implementation to the upper bounds, without
establishing empirical tightness.

\Cref{fig:a100-complete-budgets} retains all eight Distortion budgets in all
eight settings. Oracle compression cost need not decrease monotonically:
for $K=5$ at 15 minutes, its ratio is $1.0007$ at $76/252$ retained
values and $1.0044$ at $151/252$, before reaching OPT at Full.

\subsection{Training-data and date sensitivity}
\label{app:a100-sensitivity}

For $K=4$, 15-minute requests and $m=42$, varying fitting days along
one nested sampling path gives nonmonotone Distortion costs; the
731-day result is worse than several smaller samples. Calibration still
uses 639 days in every condition, so this is a sensitivity check rather
than empirical sample-complexity scaling. The protocol is in
\Cref{app:citibike-protocol}, and the complete results are retained in
\path{supplementary/results/a100_data_sensitivity.csv}.

\Cref{tab:date-sensitivity} combines the eight mean-predictor settings.
Excluding the first ten evaluation dates leaves the observed costs close
to the full-year values.

\begin{table}[htbp]
\centering\small
\setlength{\tabcolsep}{4pt}
\caption{Mean daily ALG/OPT at approximately 20\% after excluding the first ten 2025 dates. All eight settings use the same fitted models and landmarks as their full-year evaluations.}\label{tab:date-sensitivity}
\begin{tabular}{@{}lrr@{}}
\toprule
Setting & 365-day Distortion & 355-day Distortion \\
\midrule
$K=3$, 1 min & 1.2695 & 1.2707 \\
$K=3$, 15 min & 1.1659 & 1.1688 \\
$K=4$, 1 min & 1.1786 & 1.1797 \\
$K=4$, 15 min & 1.2043 & 1.2054 \\
$K=5$, 1 min & 1.0499 & 1.0494 \\
$K=5$, 15 min & 1.2737 & 1.2743 \\
$K=6$, 1 min & 1.0548 & 1.0539 \\
$K=6$, 15 min & 1.3349 & 1.3359 \\
\bottomrule
\end{tabular}
\end{table}

\FloatBarrier

\subsection{Synthetic fixed-table learning and cost certificates}
\label{app:synthetic-summary}
\label{app:certificate-results}

\Cref{tab:pilot_results} reports the two fixed-table pilots under the
data and fitting protocol in \Cref{app:data-generation,app:implementation-details}.

\begin{table}[htbp]
    \centering
    \caption{Test total episode excess $\mathrm{ALG}-\mathrm{OPT}$:
    mean $\pm$ sample standard deviation of the three seed-level means
    (not confidence intervals). Oracle full-state excess is below
    $10^{-15}$ in magnitude and displayed as zero. Bold marks the lowest
    non-oracle mean in each column, including ties, without implying
    statistical significance.}
    \label{tab:pilot_results}
    \small
    \begin{tabular}{@{}lcc@{}}
        \toprule
        Method & Localized & Switching \\
        \midrule
        Uniform random pairs (exact average)
            & $0.40154\pm0.01142$ & $0.43569\pm0.00668$ \\
        Geometric pair
            & $0.17838\pm0.01720$ & $0.39766\pm0.01932$ \\
        Learned distortion pair
            & $\textbf{0.04810}\pm0.00666$ & $0.13927\pm0.01661$ \\
        Learned validation pair
            & $\textbf{0.04810}\pm0.00666$ & $0.13927\pm0.01661$ \\
        Learned singleton
            & $\textbf{0.04810}\pm0.00666$ & $0.41436\pm0.00850$ \\
        Full-state fixed table
            & $\textbf{0.04810}\pm0.00666$ & $\textbf{0.04031}\pm0.00882$ \\
        \midrule
        Oracle geometric pair
            & $0.17838\pm0.01720$ & $0.43890\pm0.02160$ \\
        Oracle learned-validation pair
            & $0.04810\pm0.00666$ & $0.15341\pm0.01725$ \\
        Oracle full state & $0$ & $0$ \\
        \bottomrule
    \end{tabular}
\end{table}

The localized singleton matches the learned pair and fitted Full table;
this workload does not establish a benefit from additional landmarks.
Switching demand exhibits a nontrivial tradeoff between pair and Full prediction.

\Cref{tab:appendix-certificates} gives the two offline certificates in
\eqref{eq:recorded-certificates} for comparison with these observed costs.
The bounds can be substantially larger than realized excess. For
localized learned-validation pairs, mean excess is $0.04810$,
whereas the mean distortion and local certificates are $1.51201$
and $0.49766$, respectively. The local certificate uses only optimal
successors from the actual online states, which can explain a sharper
bound than the all-state distortion in this setting. Neither bound
universally dominates the other.

\begin{table}[htbp]
\centering
\small
\setlength{\tabcolsep}{5pt}
\caption{Mean offline certificates over the three seed-level test means. Observed excess and seed standard deviations are in \Cref{tab:pilot_results}; $C_{\kappa}$ and $C_{\rm loc}$ are defined in \eqref{eq:recorded-certificates}.}
\label{tab:appendix-certificates}
\begin{tabular}{@{}lrrrr@{}}
\toprule
& \multicolumn{2}{c}{Localized} & \multicolumn{2}{c}{Switching} \\
\cmidrule(lr){2-3}\cmidrule(l){4-5}
Method & $C_{\kappa}$ & $C_{\rm loc}$ & $C_{\kappa}$ & $C_{\rm loc}$ \\
\midrule
Geometric pair & 3.05480 & 1.72344 & 3.94197 & 4.21647 \\
Learned distortion pair & 1.15931 & 0.60406 & 2.13124 & 1.79679 \\
Learned validation pair & 1.51201 & 0.49766 & 2.13124 & 1.79679 \\
Learned singleton & 1.30334 & 0.28874 & 4.05037 & 3.42253 \\
Full-state fixed table & 0.94929 & 0.94929 & 0.98307 & 0.98307 \\
\midrule
Oracle geometric pair & 2.16176 & 0.83040 & 3.00155 & 3.27604 \\
Oracle learned-validation pair & 1.30308 & 0.28874 & 1.67918 & 1.34505 \\
Oracle full state & 0 & 0 & 0 & 0 \\
\bottomrule
\end{tabular}
\end{table}

For a full-state fitted table, exact reconstruction distortion is
zero and every optimal successor is a landmark. Its two recorded
certificates therefore agree and measure prediction error. For the
oracle full-state policy both vanish, and the recorded excess is
zero up to roundoff. In contrast, oracle compressed values can yield
larger excess than a fitted compressed table: switching mean excess
is $0.15341$ for the oracle learned-validation pair and $0.13927$
for its fitted counterpart. Exact dual values do not remove
the envelope's reconstruction bias, and the upper bounds do not
assert monotonicity of realized cost in landmark prediction error.

\FloatBarrier
\subsection{Implementation checks and relation to learning guarantees}
\label{app:empirical-scope}

Archived checks cover DP against exhaustive enumeration, causal predictor
inputs, Lipschitz outputs, terminal zero, landmark interpolation, Full
recovery, LP loss, and per-episode cost certificates. The theoretical
constructions were also checked numerically. These implementation checks
passed in the archived runs; detailed configurations and outcomes are
retained with the supplementary records.

The synthetic pilots implement the fixed-table LP and joint selection
directly. In Citi Bike, the practical Distortion
selector omits the fitted-table error term, the mean predictor does not
minimize span loss, and contextual neural predictors minimize centered MSE.
Those runs therefore do not instantiate the joint-ERM PAC guarantee.

\end{document}